\documentclass[journal]{IEEEtran}

\usepackage{mathtools}
\usepackage[flushleft]{threeparttable}
\usepackage{amsfonts}
\usepackage{wasysym}
\usepackage{multirow}
\usepackage{color}

\usepackage{amstext,amsmath,amssymb,amsthm}
\usepackage{algorithm}
\usepackage{endnotes}
\usepackage{algpseudocode}
\usepackage{xcolor}
\usepackage{hyperref}
\usepackage{rotating}
\usepackage{booktabs}
\usepackage{subcaption}
\usepackage{caption} 
\usepackage{url}

 \newcommand{\pib}{\mathcal{\pi}^{ib}}
 
 \newcommand{\piib}{\mathcal{\pi}^{ib}}
 
 \newcommand{\pibs}{\mathcal{\pi}^{bs}}
 \newcommand{\pbs}{\mathcal{\pi}^{bs}}

 \newcommand{\Tc}{T^{c}}

\newtheorem{mythm}{Theorem}

\begin{document}

\title{Adaptive Population-based Simulated Annealing for Uncertain Resource Constrained Job Scheduling}

\author{Dhananjay~Thiruvady, Su~Nguyen, Yuan Sun, Fatemeh Shiri, Nayyar Zaidi, Xiaodong Li~\IEEEmembership{Fellow,~IEEE}}

\maketitle

\begin{abstract}
Transporting ore from mines to ports is of significant interest in mining supply chains. These operations are commonly associated with growing costs and a lack of resources. Large mining companies are interested in optimally allocating their resources to reduce operational costs. This problem has been previously investigated in the literature as resource constrained job scheduling (RCJS). While a number of optimisation methods have been proposed to tackle the deterministic problem, the uncertainty associated with resource availability, an inevitable challenge in mining operations, has received less attention. RCJS with uncertainty is a hard combinatorial optimisation problem that cannot be solved efficiently with existing optimisation methods. This study proposes an adaptive population-based simulated annealing algorithm that can overcome the limitations of existing methods for RCJS with uncertainty including the pre-mature convergence, the excessive number of hyper-parameters, and the inefficiency in coping with different uncertainty levels. This new algorithm is designed to effectively balance exploration and exploitation, by using a population, modifying the cooling schedule in the Metropolis-Hastings algorithm, and using an adaptive mechanism to select perturbation operators. The results show that the proposed algorithm outperforms existing methods  across a wide range of benchmark RCJS instances and uncertainty levels. Moreover, new best known solutions are discovered for all but one problem instance across all uncertainty levels.

\end{abstract}


\IEEEpeerreviewmaketitle

\section{Introduction}

The mining supply chain in Australia is one of the largest in the world, where a significant proportion of the minerals mined are exported overseas. Mining is major contributor to the Australian economy, with the exporting minerals contributing around $\$$AUD 200 billion per annum to the economy, employing 240,000 individuals directly, and around 850,000 individuals  indirectly \cite{mining}. Due to its role in the export of raw materials, Australia has developed a reputation in building mining technologies and is one of the largest suppliers of these technologies and software, worldwide.\footnote{\url{https://www.austrade.gov.au/ArticleDocuments/2814/Mining-software-and-specialised-technologies-Industry-capability-report.pdf.aspx}}

There are numerous problems within the mining supply chain. Of significant interest, is the transportation of ore from mines to ports. In particular, transporting ore from multiple mining sites to a port where resources (trains and trucks) must be shared, is a key problem. If this problem is solved effectively, it can lead to large increases in throughput, thereby leading to substantial gains in revenue for the mining companies. In an abstract form, this problem can be formulated as a resource constrained job scheduling (RCJS) problem. The transportation of ore in batches can be viewed as jobs, while different transport modes (trains or trucks), can be viewed as resources. In the real setting, certain batches are required to be loaded on ships before others, and this aspect can be modelled by precedences between jobs. The timely arrival of ore at the ports is very important, as batches of ore that unnecessarily wait will lead to demurrage costs. This aspect is incorporated as a weighted tardiness for each job (different batches of ore being delayed lead to different costs), and the aim is to ensure the total weighted tardiness across all jobs is kept at minimum. 

In the literature, this problem is known as RCJS with a shared central resource \cite{singhweiskircher10,singhweiskircher11,singhernst10}. The complexity of RCJS makes it challenging for optimisation algorithms, and hence, several approaches have been proposed to tackle it. These include metaheuristics and matheuristics \cite{singhernst10, ernstsingh12, Thiruvady2012, Thiruvady2014CG, Thiruvady2016,  Cohen2017, nguyenacgp2018, nguyen_cor_2019, nguyen2020gprcjs, Thiruvady2020}. The problem was introduced by \cite{singhernst10}, who proposed a Lagrangian relaxation based heuristic, which was able to find good solutions and lower bounds. Singh and Weiskircher \cite{singhweiskircher10}, \cite{singhweiskircher11} consider a similar problem to RCJS that is modelled as machine scheduling with fractional shared resources. They propose collaborative approaches via agent-based modelling, which prove to be effective in tackling these problems. Ernst and Singh \cite{ernstsingh12} propose a Lagrangian relaxation based matheuristic that employs particle swarm optimisation, which showed improvements in solution quality across the RCJS benchmark dataset. Thiruvady et al. \cite{Thiruvady2012} consider a variant of RCJS with hard due dates, and they show that an ant colony optimisation (ACO) and constraint programming hybrid can efficiently find feasible and high quality solutions. The study by Thiruvady et al. \cite{Thiruvady2014CG} proposes a column generation approach for RCJS, which finds excellent solutions and lower bounds. An inherent issue is the time requirements in solving this problem, and hence, parallel implementations of metaheuristics and matheuristics have been attempted \cite{Thiruvady2016, Cohen2017}. The best known solutions for RCJS have been obtained by Nguyen et al. \cite{nguyen_cor_2019} and Blum et al. \cite{blum2019}. Nguyen et al. \cite{nguyen_cor_2019} develops a hybrid of differential evolution and column generation, and they show excellent results on small and medium sized problems. Blum et al. \cite{blum2019} devise a biased random key genetic algorithm, which produces excellent results, especially for large problem instances. In the interest of finding good solutions  quickly, studies have shown  that genetic programming can be effective for RCJS \cite{nguyenacgp2018,nguyen2020gprcjs,Nguyen2021}. Finally, Thiruvady et al. \cite{Thiruvady2020} propose Merge search and CMSA for RCJS. 

RCJS under uncertainty (RCJSU) is a variant of RCJS that takes into consideration different practical random sources such as job release times, job processing times or the availability of resources. Of vital importance in RCJS, and crucial in real-world settings, is the uncertain resource availabilities. That is, if trains and trucks breakdown or tracks and roads need maintenance leading to delays in the ore reaching ports, there can be substantial losses incurred. Therefore, planners need a solution method that can address disruptions due to random events to ensure that the solutions obtained can ensure smooth operations either by incorporating uncertainty into optimisation problems, reoptimisation, or reactive dispatching rules. Thiruvady et al. \cite{thiruvady2022} developed the first RCJSU formulation, and proposed a surrogate-assisted ant colony system (SACS) to find robust solutions. The results show that SACS outperforms solution methods based on mixed integer programming (MIP), meta-heuristics and surrogate-assisted evolutionary algorithms. Further analyses show that SACS finds good solutions more quickly compared to standard ACO algorithms.

Although SACS has shown promising results, there are some key limitations. Firstly, surrogate models help SACS address the challenge of expensive solution evaluations, but the proposed algorithm tends to converge prematurely, especially for large instances. Therefore, SACS (as well as other existing algorithms in the study of Thiruvady et al. \cite{thiruvady2022}) can waste a lot of time exploring the search space without making significant progress. Secondly, the proposed method has many pre-defined hyper-parameters such as the learning rate of ant colony systems (ACS) and pre-selection rate for the surrogate models. Because of this issue, the algorithm can have trouble in solving a wide range of instances without major modifications (e.g. reoptimise hyper-parameters). Finally, the performance of SACS is not consistent across different uncertainty levels and different problem sizes (i.e., number of machines). For medium problem sizes and large uncertainty levels, SACS does not show significantly superior performance compared to standard ACS.

This paper addresses the three key limitations discussed above by developing a new adaptive population-based simulated annealing (APSA) algorithm for RCJSU. The main contributions of this paper are summarised as follows: (1) a new population-based simulated annealing algorithm with a balance between exploration and exploitation, (2) a new efficient Metropolis-Hastings algorithm extended from the simulated annealing algorithm proposed in  \cite{singhernst10}, and (3) a new adaptive mechanism to update the probabilities for selecting perturbation operators to avoid premature convergence. APSA is similar to memetic algorithms, and population-based stochastic local search \cite{ong_memetic_2010,Hoos2015} that combines evolutionary algorithms and local search heuristics; however, APSA adopts a much smaller population in which the population is improved mainly via local search and adaptive perturbations rather than crossover operators. The balance between exploration and exploitation within APSA allows the algorithms to perform efficiently against noisy and complex problems like RCJSU. Extensive experiments are conducted to validate the superior performance of the proposed algorithms for RCJSU and RCJS.

The paper is organised as follows. Section~\ref{rel_work} provides a review of related work. In Section~\ref{sec:problem}, a formulation of the RCJSU problem is provided, including an integer programming model.  Section~\ref{sec:methods} discusses the adaptive population-based simulated annealing algorithm proposed in this study. The experimental evaluation and associated results are presented in Section~\ref{sec:results}. Section~\ref{sec:conclusion} concludes the paper and presents possibilities for future work.  

\section{Related Work} \label{rel_work}
A review of the literature related to RCJS is provided here, including methods and techniques that are effective at tackling problems similar to RCJS and RCJSU. 

\subsection{Optimisation techniques for tackling RCJS}

Owing to the complexity of RCJS, mathematical programming methods on their own do not suffice in finding optimal or even feasible solutions to relatively small problems. Hence, metahuristics and matheuristics have been popular approaches for tackling the problem. Metaheuristics, including partical swarm optimisation \cite{ernstsingh12}, ACO \cite{Thiruvady2016}, differential evolution \cite{nguyen_cor_2019}, and genetic algorithms \cite{blum2019} have been known to find good solutions, albeit with large time overheads. RCJS can be modelled with graphs, and hence, ACO lends itself well to the problem. Nonetheless, past studies on RCJS have shown that ACO tends to converge slowly and gets trapped in local optima. Moreover, diversification strategies including local search (e.g., parallel implementations or niching techniques) help ACO, but the challenges still remain. Hybrid methods, especially those that can guarantee optimality, have also been explored \cite{singhernst10, Thiruvady2014CG, nguyen_cor_2019}, but the results show that only relatively small RCJS problem instances (up to 30 jobs) can be solved with these methods.

\subsection{Scheduling on problems with uncertainty}\label{sec:klit_review}

Uncertainty is inherent in many real-world applications, and dealing with uncertainty is often challenging when devising scheduling algorithms. If uncertainty is not considered, the ensuing scheduling algorithms in the deterministic setting are often not effective at even finding feasible solutions  (e.g. if a machine breaks down, a job will not be able to start at its pre-determined start time). Hence, a good understanding of the setting in which the schedules are applied is important to ensure success of the scheduling algorithms. Often, reactive scheduling strategies can be more effective in dealing with frequently occurring situations with uncertainties, since the latest information can be incorporated into the decisions. Nonetheless, if predicting the schedules is needed and knowledge of uncertainty is available, the preference is to develop robust schedules. The focus of this paper is on finding robust schedules for RCJSU, since predictable schedules are necessary for mining operations to coordinate supporting activities and to ensure that ore arrives at ports in a timely manner. Of particular interest is the timely arrival of ore at the ports, which relies greatly on the availability of resources. This is indeed the focus in the RCJSU problem and presented in the formulation in Equations \eqref{obj} to \eqref{m2c6} (see Section~\ref{sec:problem}).

When tackling scheduling problems with uncertainty, the aim is typically to devise methods that are capable of identifying good solutions under any uncertain scenario. These solutions are called robust solutions and the research area, robust proactive scheduling under uncertainty. There are numerous studies focusing on this area, which we summarise here. Beck and Wilson \cite{Beck11} investigate job shop scheduling, focusing on a variant which consists of stochastic processing times. Their results demonstrate that hybrids of tabu search and constraint programming with Monte Carlo simulation are effective in handling large-sized problems with uncertainty. The study by Song et al. \cite{Song19}, proposes branch-and-bound algorithms that generate robust schedules on a variant of resource-constrained project scheduling with processing times that are uncertain depending on time. Identifying robust solutions in uncertain problems introduces additional complexity, and for this purpose metaheuristcs such as swarm intelligence algorithms and evolutionary approaches are known to be effective. Bierwirth and Mattfield \cite{Bierwirth99} study genetic algorithms for dynamic job arrivals in classic job shop scheduling. Their results show that their approaches are superior to priority dispatching rules, which are popular in dynamic production scheduling. Wang et al. \cite{Wang201570} devise a surrogate-assisted multi-objective evolutionary method to tackle proactive scheduling where machines are considered to have stochastic breakdowns. To cope with the uncertainty, they propose a job compression and rescheduling strategy \cite{Gao2019}. Support vector regression in conjunction with surrogate models have been proposed to enhance the speed of time-intensive simulations. The study by Wang et al. \cite{wang20} provides an in-depth analysis and discussion of scheduling with machine breakdowns. Systematic reviews of scheduling in problems with uncertainty have also been conducted \cite{chaari2014scheduling,Ouelhadj2008}.

The least two decades have seen numerous studies that consider uncertainty in scheduling problems similar to that of RCJS. There are several similarities to project scheduling, and many studies have considered uncertainty within this setting \cite{herroelen2005project,demeulemeester2011robust,lambrechts2011time,masmoudi2013project,Khodakarmi2007,moradi2019robust,CHAKRABORTTY2017537}. In these studies, uncertainty is typically considered in respect to the resources, though, a few consider processing time related uncertainty. Lambrechts et al. \cite{lambrechts2011time} study uncertainty with resources, and specifically modes of execution related to the activities. They develop the notion of ``time buffering,'' that efficiently generates robust schedules. Masmoudi and Ha\"{i}t \cite{masmoudi2013project} consider resource levelling and resource constrained project scheduling, and they propose fuzzy modelling to tackle uncertainty. Their solution approaches are a genetic algorithm and a greedy method to identify robust schedules. Khodakarami et al.  \cite{Khodakarmi2007} investigate Bayesian networks to identify causal structures considering uncertainty in sources and parameters of a project. They find that Bayesian networks are indeed very effective at dealing with uncertainty, including aspects such as representation, elicitation and management. Uncertainty associated with processing times has also been investigated \cite{moradi2019robust,CHAKRABORTTY2017537}. Moradi and Shadrokh \cite{moradi2019robust} develop a problem-specific heuristic to deal with uncertainty, and the study by Chakrabortty et al. \cite{CHAKRABORTTY2017537} propose several Branch \& Cut heuristics for project scheduling.\footnote{Benchmark datasets obtained from the project scheduling problem library (PSPLIB) \cite{KOLISCH1997205}.} 

\subsection{Hybrid optimisation methods} 
Local search and population-based methods \cite{Luke2013Metaheuristics} have been extensively investigated in the literature to tackle challenging scheduling problems. While local search heuristics are best known for their ability to find local optima efficiently, population-based optimisation methods such as genetic algorithms are better known for their exploration ability. Empirical studies have demonstrated the need to maintain a balance between exploration and exploitation when dealing with difficult combinatorial optimisation problems such as job shop scheduling \cite{CHENG1999343}. Such hybrid methods are commonly referred to as memetic algorithms \cite{ong_memetic_2010} or population-based stochastic local search algorithms \cite{Hoos2015} in the literature. While hybrid algorithms are effective in  combinatorial optimisation, there are still a few shortcomings. Firstly, design choices to develop  competitive hybrid algorithms are difficult to make and can be sensitive to different problem subsets. Secondly, the performance of hybrid algorithms depend strongly on an ``ideal'' setting of the algorithm parameters.

Hybrid algorithms have also been developed for RCJS. The study that introduced the problem \cite{singhernst10} develop a Lagrangian relaxation heuristic, which combines MIP and simulated annealing. Following this, Thiruvady et al. \cite{Thiruvady2014CG} show that a column generation and ACO hybrid is very effective on this problem by finding good lower bounds. Parallel implementations of ACO \cite{Thiruvady2016} have also been effective, by finding high-quality solutions quickly. \cite{nguyen_cor_2019} propose a hybrid algorithm combining differential evolution and local search with excellent results. The study by Thiruvady et al. \cite{Thiruvady2020} show that parallel matheuristics based on MIP and ACO can effectively tackle this problem.

Thiruvady et al. \cite{thiruvady2022} has developed SACS, a sophisticated algorithm combining ACO, surrogate modelling, and local search to deal with RCJSU. The experiments show that SACS can find excellent solutions for small problems thereby outperforming other meta-heuristics. However, further analyses show that SACS tends to converge prematurely and is inconsistent when working with different problem subsets. Moreover, SACS has many hyper-parameters (for ACO, local search, surrogate models), for which the ideal settings are difficult to find.

\section{Problem Formulation}
\label{sec:problem}

The formal definition of RCJSU as an integer program (IP) was provided by Thiruvady et al. \cite{thiruvady2022}. Past studies have considered uncertainty in different ways, including resource limits \cite{lambrechts2011time} or uncertainty in job/task processing times \cite{CHAKRABORTTY2017537}. However, in this study the focus is on uncertainty derived from the variability in resources, which in turn affects the processing times of jobs/tasks.\footnote{We refer the reader to the study by Thiruvady at al. \cite{thiruvady2022} for a detailed discussion of uncertainty in respect to this problem}. The problem is NP-hard as it is at least as hard as its deterministic counterpart, which is indeed shown to be NP-hard \cite{singhernst10}. This IP formulation extends the one of Singh and Ernst \cite{singhernst10}, proposed for RCJS. Uncertainty is introduced by considering several uncertain resource levels, that is, some proportion of the original resource limit is used for each uncertainty sample (details in Section~\ref{sec:results}). 

We are given several jobs $\mathcal{J}$ $= \{1,\ldots,n\}$, which are associated with the following information. Each job $i$ must execute on machine $m_i$, ($\mathcal{M}$ $= \{1,\ldots,l\}$), has a a release time $r_i$, processing time $p_i$, due time $d_i$, weight $w_i$ and resource usage $g_i$. Jobs on the same machine cannot execute concurrently, or in other words, only one job may execute at one time. A time horizon is given, consisting of discrete points ${\mathcal{T}}=\{0,\ldots,D\}$, and $D$ is chosen to be sufficiently large to ensure a feasible solution will always be found. Precendences may exist between any two jobs $i \in \mathcal{J}$ and $j \in \mathcal{J}$ that execute on the same machine. This is represented as $i \rightarrow j$, where job $i$ completes before job $j$ starts. 

A central resource constraint exists across all machines, where the cumulative resource usage of jobs executing at any time point must be at most $\mathcal{G}$. When considering uncertainty, $\mathcal{G}$ is allowed to vary. This leads to defining a {\it sample}, which has available a proportion of the resource $\mathcal{G}$, and is an instance of the problem. As an input, the total number of samples $u$ is provided, and the resource limit is $\mathcal{G}_s$ applies to the $s^{th}$ sample ${\cal U} = \{1,\ldots,u_S\}$.

The IP can be defined as follows. Let $z_{sjt}$ be a binary variable, which indicates in sample $s$, job $j$ completes or has completed by time $t$, if $z_{sjt} = 1$. By this definition, a job stays complete once completed. Compared to RCJS, the RCJSU IP formulation requires additional space complexity, which consists of ${\mathcal O}(n^3)$ variables, and each set of constraints with an additional dimension owing to the samples. The IP formulation is: 

\begin{alignat}3
 \min \quad \frac{\sum_{s \in \cal U}\sum_{j \in \cal J}\sum_{t \in \cal T} w_{jt} \times (z_{sjt} - z_{sjt-1})}{u_s}  & \label{obj}
 \end{alignat}
\begin{alignat}3
&\text{s.t. }  \nonumber \\ 
& z_{sjt_{\max}} = 1  & \forall\ s \in {\cal U}, \forall\ j \in {\cal J}  \label{m2c1} \\
& z_{sjt} - z_{sjt-1} \geq 0 & \forall\ j \in {\cal J}, \forall\ s \in {\cal U}, \nonumber \\ 
&& t \in \{1,\ldots,t_{\max}\} \label{m2c2} \\
& z_{sjt} = 0  & \forall\ t \in \{1,\ldots,e_j\}, \nonumber \\
&& \forall\ s \in {\cal U},\ \forall\ j \in {\cal J}  \label{m2c3} \\
& z_{sbt} - z_{sa,t-p_b} \leq 0  & \forall\ (a,b) \in  {\cal C},\ \forall\ s \in {\cal U}, \nonumber \\ 
&& \forall\ t\in {\cal T}  \label{m2c4} \\
& \sum_{j \in {\cal J}^i} z_{sj,t+p_j} - z_{sjt} \leq 1 &  \forall\ i \in {\cal M},\ \forall\ s \in {\cal U}, \nonumber \\ 
&& \forall\ t\in {\cal T} \label{m2c5} \\
& \sum_{j \in {\cal J} } g_j \cdot (z_{sj,t+p_j}-z_{sjt}) \leq {\cal G}_s & \forall\  s \in {\cal U},\ \forall\ t \in  {\cal T}  \label{m2c6}
\end{alignat}
Equation~\eqref{obj}, the objective function, states that the average total weighted tardiness (TWT) across all samples must be minimised. Constraints~\eqref{m2c1} ensure that all jobs complete. Once a job completes it stays complete, which is enforced by Constraints~\eqref{m2c2}. Constraints~\eqref{m2c3} ensure that all jobs satisfy release times and the precedences are imposed by Constraints~\eqref{m2c4}. In this constraint, $e_j = r_j+p_j-1$. Constraints~\eqref{m2c5} require that only one job can execute on one machine at every time point. Constraints~\eqref{m2c6} ensure that the resource limits, considering each sample, are satisfied.

For the RCJS, there are efficient scheduling heuristics that are known (see Section~\ref{sec:heur}). These heuristics typically depend on an alternative formulation, where a sequence or permutation jobs are efficiently mapped to schedules. Previous studies have used this approach for RCJS (e.g. \cite{Thiruvady2014CG}), and it has also proved beneficial for the RCJSU \cite{thiruvady2022}. In this study, $\pi$ represents a sequence of jobs while $\sigma(\pi)$ represents an efficient mapping from $\pi$ to a feasible schedule. Therefore, $\pi$ is effectively a solution to the problem, where given a sequence, a deterministic schedule can be efficiently constructed.

An example RCJS problem instance is shown in Table~\ref{tab:example}. This example shows how a solution is constructed and what the optimal solution is. The problem consists of three jobs, with the associated information in the tables. Moreover, the resource capacity is 10 units. Considering $\pi$, there are six potential sequences. Among the jobs, the resulting schedules from $\pi$ will be efficient if Job 3 is scheduled early owing to an early due date and  high weight. Additionally, Job 2 should be sequenced before Job 1 due to its higher weight. Given this information, the optimal solution is $\pi = \{3,2,1\}$. 

\begin{table}[ht]
\centering
\caption{An example with 3 jobs and 10 units of resource. Proc. is the processing time.}
\label{tab:example}
 \scalebox{0.8}{
\begin{tabular}{llrrrrrrr}
\toprule			

 Job && Resource && Weight && Proc. && Due  \\
\midrule
    1 && 5 && 0.1 && 1 && 2 \\ 
     2 && 10 && 0.2 && 1 && 2 \\ 
      3 && 10 && 0.5 && 1 && 1 \\ 
\bottomrule
\end{tabular}
}
\end{table}

\section{Population-based Simulated Annealing}\label{sec:methods}

Several metaheuristics have been able to find excellent solutions to RCJS in a time efficient manner. One of these approaches was simulated annealing ~\cite{singhernst10}. However, to tackle RCJSU, the traditional simulated annealing method has a few shortcomings, and hence, we propose an adaptive  population-based simulated annealing (APSA). In particular, this method has three main differences from the traditional approach: (1) a population of solutions to track and exploit best known solutions, (2) a faster cooling schedule within the Metropolis-Hastings algorithm, and (3) adaptive probabilistic selection of the neighbourhood moves to perturb solutions in the population. Comparisons are made to ACS (ant colony system) and SACS (surrogate-assisted ant colony system) of \cite{thiruvady2022}, which have previously been effective on RCJSU.  

\subsection{The APSA Algorithm}\label{sec:alg}

Algorithm~\ref{alg:psa} shows the high-level APSA procedure for RCJSU. As input, the algorithm takes the following: (a) $T_0$: an initial temperature, (b) $t_{lim}$: time limit for the execution of the algorithm, (c) $Iter$: number of Metropolis-Hastings iterations, (d) $s$:  population size, (e) $\gamma$: a parameter to control the cooling schedule and (f) $\rho$: a parameter that specifies how to adapt the probabilities ($\rho =$ 0.9 for this study). In Line~2, the population $\Pi$ is initialised, where each solution consists of a random list of all the tasks. The best known solution, $\pibs$, is initialised to be empty. In Line~3, the probability distributions associated with three neighbourhood moves (discussed in the following) are initialised. The main procedure takes place between Lines~4--21, and the algorithm completes execution when the time limit is reached. 

\begin{algorithm}[t]
\caption{Adaptive Population-based Simulated Annealing (APSA)}
\label{alg:psa}
\begin{algorithmic}[1]
  \State {\bf Input:} An MMALBPS instance, $T_0$, $t_{lim}$, $Iter$, $s$, $\gamma$, $\rho$
  \State {\sf $\Pi \leftarrow$ Initialise}($s$), {\sf $\pbs \leftarrow \emptyset$}
  \State $p^b \leftarrow 0.65$, $p^j \leftarrow 0.3$, $p^r \leftarrow 0.05$
\While {execution time $< t_{lim}$}
    \For{$\pi \in \Pi$}
    \State {\sf $\pi^t :=$ MetropolisHastings($\pi$, $T_0$, $Iter$, $\gamma$)}
      	\State $\pbs :=$ {\sf Update($\pi$) }
    \State $r \leftarrow$  $\text{{\sf rand}}()$
     \If { $r < p^b$} 
        \State {\sf $\beta-$Sampling($\pbs$)}
        \State $p^b \leftarrow p^b \times \rho$ 
    \ElsIf {$r <  p^b + p^j$} 
        \State $\pi := $ {\sf SwapOps($\pib$, $m$) }
        \State $p^j \leftarrow p^j \times \rho$
    \Else
        \State $\pi \leftarrow $ {\sf RandomList()}
        \State $p^r \leftarrow p^r \times \rho$
    \EndIf  	
    \State {\sf Normalise($p^b$, $p^j$, $p^r$) }
    \EndFor
\EndWhile
\State $\pi^f :=$ {\sf SwapAllJobPairs($\pbs$)}
\State return $\pi^f$
\end{algorithmic}
\end{algorithm}

Within the main loop, the following is applied to each solution in the population. First, in Line~6, the Metropolis-Hastings algorithm is applied (Algorithm~\ref{alg:mh}) using $\pi$ as the best solution and for $Iter$ iterations. The resulting solution is stored in $\pi^t$. 
$\pibs$ is updated to $\pi$, if $\pi$ is an improvement, i.e., $f(\pi) < f(\pibs)$. Following this, between Lines~9--19, $\pi$ is perturbed using a neighbourhood move (details follow in Section~\ref{sec:nm}). As part of these steps, the probabilities associated with choosing one of the perturbation moves are adapted or updated. That is, if a move is selected, its corresponding probability reduced by an adaptation factor $\rho = 0.9$.  

Every solution obtained (by applying any neighbourhood move, etc.) requires the computation of its objective value. As mentioned earlier, a solution is represented by a sequence of the tasks. For each sequence of tasks $\pi$, a feasible schedule ($\sigma(\pi)$) is obtained by constructing $u_s$ schedules corresponding to each uncertainty sample. Note, each schedule is resource feasible, ensuring that at most ${\cal G}_s$ resources are used at each time point for sample $s$. The result is that each sample has its own TWT. The objective value of $\pi$ is computed as the average of the TWTs across all samples. 

\subsection{The Metropolis-Hastings Algorithm}

\begin{algorithm}[t]
\caption{{\sf MetropolisHastings}}
\label{alg:mh}
\begin{algorithmic}[1]
  \State {\bf Input:} $\pi$, $T_0$, $Iter$, $\gamma$
    \State $C \leftarrow \infty$, $\Tc \leftarrow T_0$, $\pi^{ib} \leftarrow \pi$ 
      \For {$i = 1$ to $Iter$}
      	\State $\pi^b \leftarrow$ {\sf SwapJobs($\pi^{ib}$)}
      	\If{$f(\pi^b) < C$}
      	    \State{$C \leftarrow f(\pi^b)$}
      	    \State $\pi^{ib} \leftarrow \pi^b$
      	    \State $i \leftarrow 1$
      	\Else
      	    \State $\Delta = f(\pi^{ib}) - C$
      	    \If{$e^{-\Delta/\Tc} \leq  \text{{\sf rand}}()$} 
      	        \State $\pi^{ib} \leftarrow \pi^b$
      	    \EndIf
      	    \State $\Tc \leftarrow \Tc \times \gamma$ 
      	\EndIf   
      	\EndFor
\State return $\pi^{ib}$
\end{algorithmic}
\end{algorithm}

The Metropolis-Hastings method relies on the cooling schedule to ensure a gradual convergence to high-quality regions of the search space. However, due to the overheads introduced by uncertainty, the traditional gradual reduction in temperature will not suffice, and hence we propose a modified version of the method for RCJSU.

Algorithm~\ref{alg:mh} shows the  Metropolis-Hastings procedure. It takes as input: (a) $\pi$: current best solution, (b) $T_0$: initial temperature, (c) $Iter$: number of iterations that the algorithm will execute for and (d) $\gamma$: a parameter to determine the cooling schedule. $C$ represents the best known objective value, and is set to a large number. The temperature $T^c$ is set to the initial temperature $T_0$, and the best solution $\piib$ is initialised to $\pi$. The algorithm executes for $Iter$ iterations between Lines~3--16. Using $\piib$, the neighbourhood move of swapping a pair of jobs is applied. That is, two jobs within the sequence are selected at random and swapped (see Figure~\ref{fig:job_swap}). If the new solution found $\pi^b$ is an improvement, $C$ is updated to the objective value of $\pi^b$. Moreover, $\piib$ is updated to $\pi^b$ and the iteration count is reset. Otherwise, a move to the non-improving solution $\pi^b$ is still chosen with probability $e^{-(f(\pi^{ib}) - C)/\Tc}$. The cooling schedule takes place in Line~14.

\subsection{Adaptive Perturbation}\label{sec:nm}

Key to the success of APSA is perturbing the best known solution, via three neighbourhood moves. In particular, when applied carefully at different stages of the search, the overall APSA algorithm is able to achieve excellent trade-offs between intensification and diversification. 

The first neighbourhood move is $\beta$-sampling, which was first studied in the context of project scheduling \cite{valls03}. The basic idea underlying the method is that consecutive jobs or tasks with a sequence are inter-dependent, and moving whole sub-sequences (maintaining certain dependencies) can lead to significant improvements in solutions. For the purposes of the current study, starting at a randomly selected position in the sequence, a sub-sequence of fixed length (5 jobs) is moved to the back of the permutation.

Figure~\ref{fig:beta_samp} provides an illustration of this method. $\pi$ consists of 11 jobs. A position $p \in {1,\ldots,n-5}$ is chosen randomly, and this sub-sequence ($p$ to $p+5$) moves to the end (bottom part of the figure). Here, $p = 4$, which starts at Job 7, and Jobs 7, 9, 10, 3 and 11 are chosen to move to the end of the sequence. Subsequent jobs (Jobs 8,4, and 5) move up in the sequence to position 4. $\beta$-sampling is applied in the above manner in Line~10 of Algorithm~\ref{alg:psa}. 

\begin{figure}[t]
\centering
\includegraphics[clip,width=0.4\textwidth]{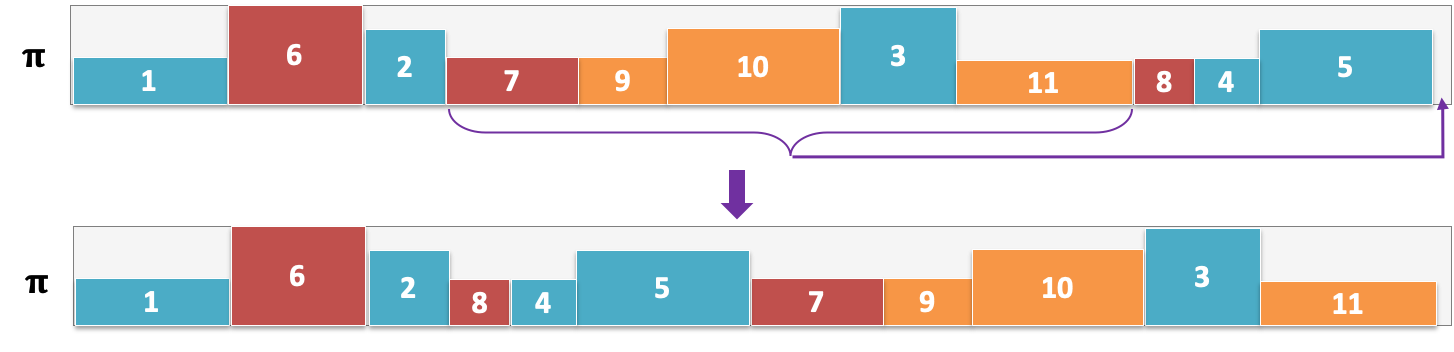}
\caption{An example that demonstrates $\beta$-sampling. From the sequence $\pi$, a subset of jobs is selected (7, 9, 10, 3, 11) and moved to the end of the sequence.} 
\label{fig:beta_samp}
\end{figure}

The second neighbourhood move is Job swapping. This consists of swapping pairs of jobs in $\pi$. Figure~\ref{fig:job_swap} shows an example of swapping one pair of jobs. Two positions or indices are chosen at random in the sequence $\pi$, and the jobs in these positions are swapped. In the example, Jobs 2 and 8 are swapped. 

\begin{figure}[t]
\centering
\includegraphics[clip,width=0.4\textwidth]{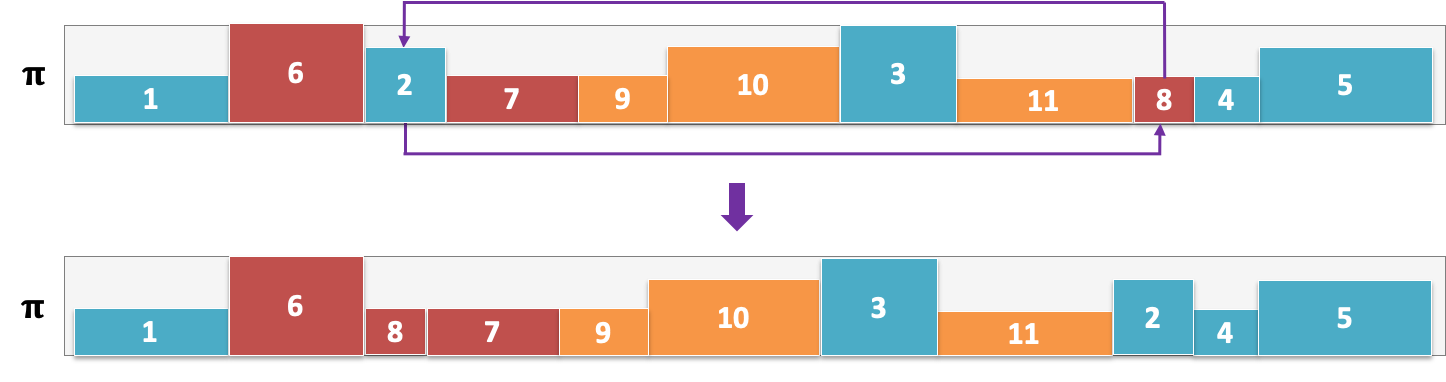}
\caption{An example that demonstrates job swapping. Two indices are selected randomly, and the corresponding jobs are swapped. In this example, indices 3 and 9 are selected, and the jobs within those indices, 2 and 9, are swapped.} 
\label{fig:job_swap}
\end{figure}

The third neighbourhood move is to re-initialise the current solution to a randomly chosen sequence. Not surprisingly, the resulting solution can belong to vastly different areas of the search space with a very different objective function value. Hence, this neighbourhood move is expected to produce the largest amount of diversification. 

\subsubsection{Convergence of Probabilities} \label{sec:prob_conv}

Assuming the probability values of $p_b$, $p_j$ and $p_r$ at time step $t$ are $p_b^t$, $p_j^t$ and $p_r^t$, the expected probabilities at time step $t+1$ can be calculated as
\begin{align}
    p_b^{t+1} & = p_b^t \Bigg(\frac{0.9p_b^t}{1-0.1p_b^t}\Bigg) + p_j^t \Bigg(\frac{p_b^t}{1-0.1p_j^t}\Bigg) + p_r^t \Bigg(\frac{p_b^t}{1-0.1p_r^t}\Bigg), \\ 
    p_j^{t+1} & = p_j^t \Bigg(\frac{0.9p_j^t}{1-0.1p_j^t}\Bigg) + p_r^t \Bigg(\frac{p_j^t}{1-0.1p_r^t}\Bigg) + p_b^t \Bigg(\frac{p_j^t}{1-0.1p_b^t}\Bigg), \\
    p_r^{t+1} & = p_r^t \Bigg(\frac{0.9p_r^t}{1-0.1p_r^t}\Bigg) + p_b^t \Bigg(\frac{p_r^t}{1-0.1p_b^t}\Bigg) + p_j^t \Bigg(\frac{p_r^t}{1-0.1p_j^t}\Bigg). 
\end{align}
Hence the governing differential equations for evolving the probabilities $p_b$, $p_j$ and $p_r$ can be derived:
\begin{align}
    \frac{dp_b}{dt} & = 0.1p_b \Bigg( \frac{p_j^2}{1-0.1p_j} + \frac{p_r^2}{1-0.1p_r} - \frac{p_b(p_j+p_r)}{1-0.1p_b}\Bigg), \label{eq:pb}\\ 
    \frac{dp_j}{dt} & = 0.1p_j \Bigg( \frac{p_r^2}{1-0.1p_r} + \frac{p_b^2}{1-0.1p_b} - \frac{p_j(p_b+p_r)}{1-0.1p_j}\Bigg), \label{eq:pj}\\
    \frac{dp_r}{dt} & = 0.1p_r \Bigg( \frac{p_b^2}{1-0.1p_b} + \frac{p_j^2}{1-0.1p_j} - \frac{p_r(p_b+p_j)}{1-0.1p_r}\Bigg). \label{eq:pr}
\end{align}

\begin{mythm}
For any initial condition of $p_b^0$, $p_j^0$, $p_r^0 \in [0,1]$ and $p_b^0 + p_j^0 + p_r^0 = 1$, the probability values converge to $p_b = p_j = p_r = 1/3$.
\end{mythm}

\begin{proof}
At equilibrium of the differential equations \eqref{eq:pb}-\eqref{eq:pr}, $\frac{dp_b}{dt} = \frac{dp_j}{dt} = \frac{dp_r}{dt} = 0$. Solving this system of equations with $p_b + p_j + p_r = 1$, we can obtain the following seven equilibria from the differential equations: 
\begin{itemize}
    \item ($p_b=0, p_j=0, p_r=1$),
    \item ($p_b=0, p_j=1, p_r=0$),
    \item ($p_b=1, p_j=0, p_r=0$),
    \item ($p_b=0, p_j=1/2, p_r=1/2$),
    \item ($p_b=1/2, p_j=0, p_r=1/2$),
    \item ($p_b=1/2, p_j=1/2, p_r=0$),
    \item ($p_b=1/3, p_j=1/3, p_r=1/3$).
\end{itemize}
However, the first six equilibria are unstable, in the sense that a small disturbance will perturb the system away from those equilibria. The last equilibrium ($p_b=1/3, p_j=1/3, p_r=1/3$) is the only stable equilibrium. Hence, after a sufficient number of time steps, the probability values will converge to this stable equilibrium. 
\end{proof}

Figure~\ref{fig:my_label} shows a simulation of how the probabilities associated with each neighbourhood move (Algorithm~\ref{alg:psa}) evolve across 200 simulations. We see that three  probabilities converge to almost equal levels by 125 simulations. 

\begin{figure}
    \centering
    \includegraphics[clip,width=0.4\textwidth]{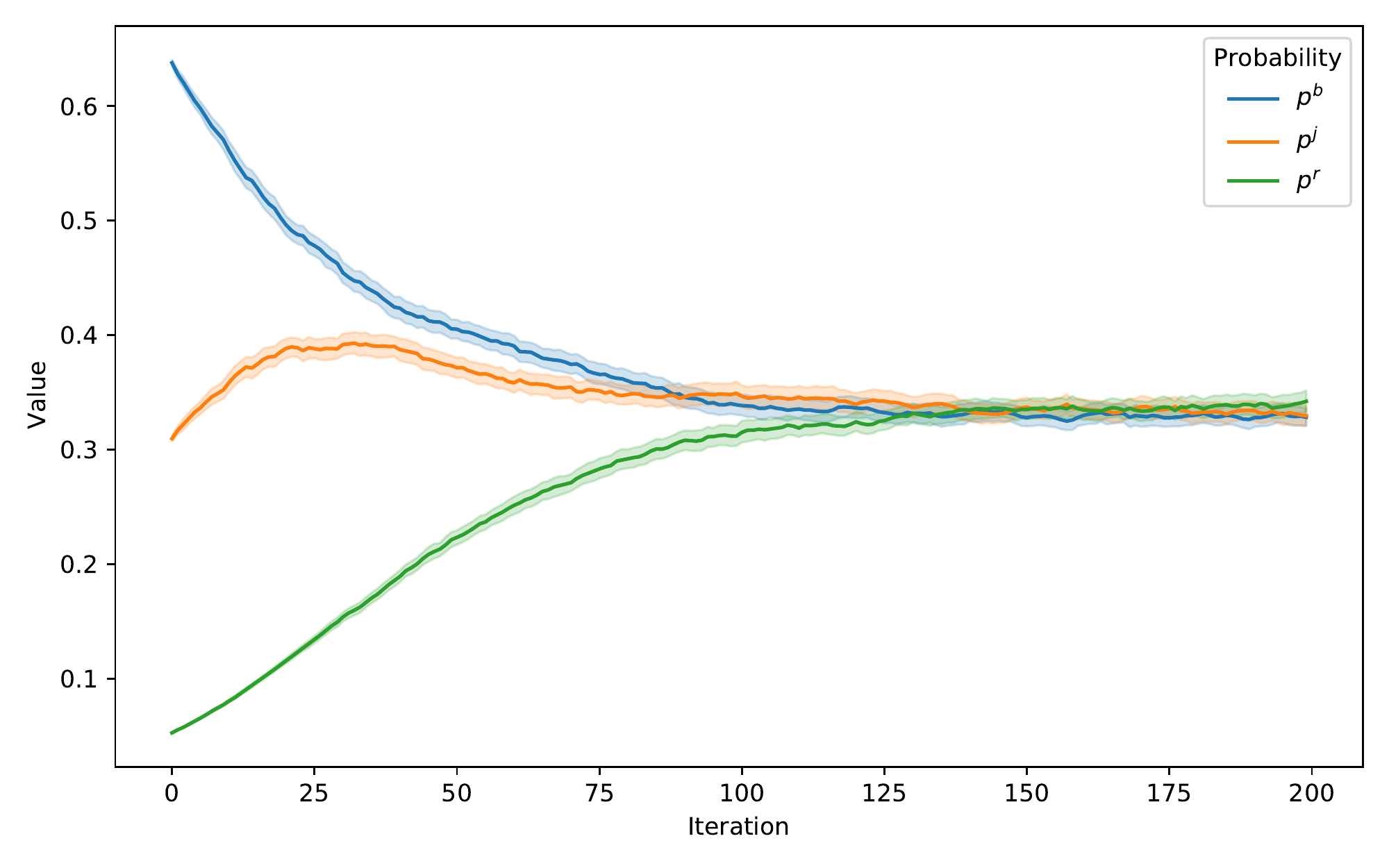}
    \caption{Adaptive probabilities over 200 simulations.}
    \label{fig:my_label}
\end{figure}

\subsection{Scheduling Heuristic}\label{sec:heur}

The aim of the scheduling heuristic is to generate a resource feasible schedule $\sigma(\pi)$ from the sequence $\pi$. For this purpose, the serial scheduling scheme is employed \cite{kolish96}, which is adapted from previous studies on RCJS \cite{singhernst10,Thiruvady2012,Thiruvady2014CG, Thiruvady2016}. Note, Thiruvady et al. \cite{thiruvady2022} showed that this scheduling scheme can indeed be effective for RCJSU.

\begin{figure}[htbp]
\centering
\subfloat[]{\label{subfig:a}\includegraphics[clip,width=0.4\textwidth]{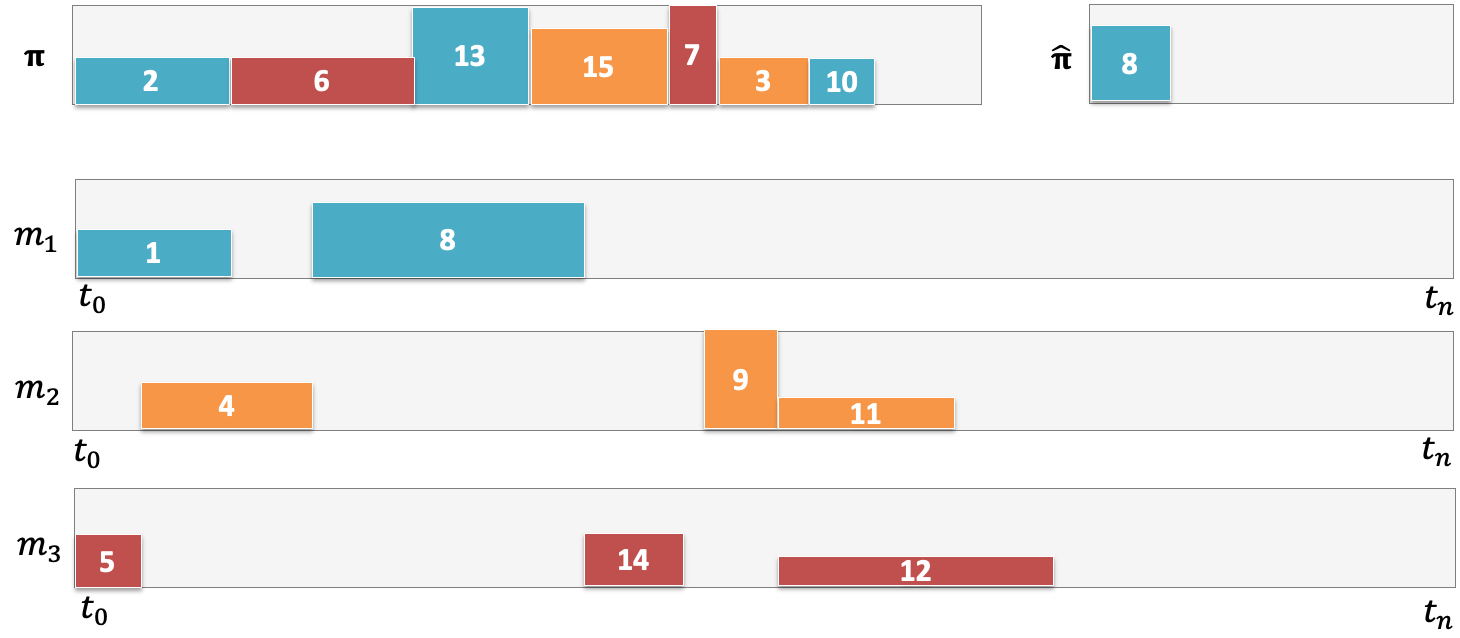}}\\
\subfloat[]{\label{subfig:b}\includegraphics[clip,width=0.4\textwidth]{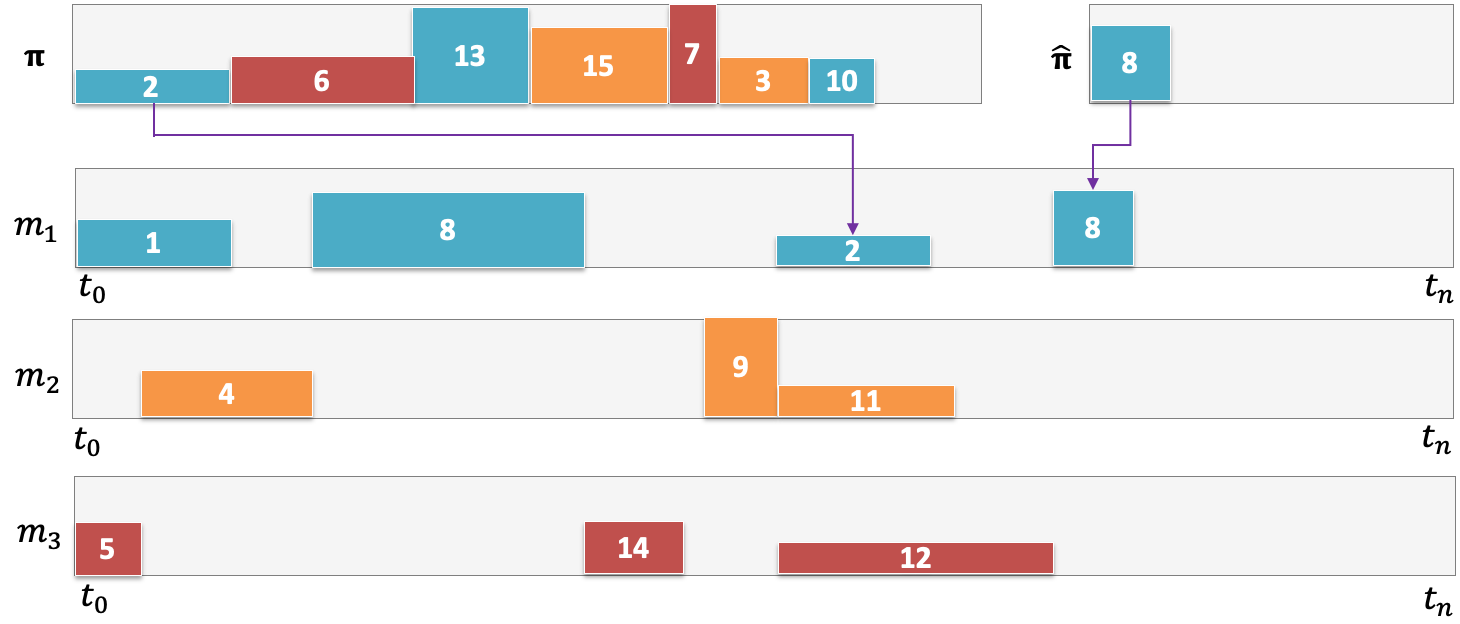}}
\vspace{-0.5em}
\caption{An example of a problem with three machines and 15 jobs. From $\pi$, jobs are selected in order (left to right) and scheduled on machines that they belong to (e.g. orange jobs on $m_2$). Those jobs that have predecessors that are not yet scheduled will be put on the waiting list $\hat{\pi}$; (a) Job 8 requires Job 2 to be scheduled before it can be scheduled. (b) Job 8 is scheduled once Job 2 has been scheduled, and this will be done as early as possible considering resources.} 
\label{fig:15job}
\end{figure}

The scheduling heuristic works as follows. Jobs in $\pi$ are selected in the order in which they appear. Each job ($j$) is tested to see if any of its preceding jobs have not yet been scheduled, and if so, $j$ is placed on the waiting list. Otherwise, $j$ is scheduled as early as possible to ensure that the resource limits are not violated. Once $j$ is scheduled, any job that depends on $j$ (and only $j$) on the waiting list is scheduled immediately. All such jobs on the waiting list are scheduled in order. Note, for RCJSU, schedules using the procedure above are created for each sample. 

Figure~\ref{fig:15job} shows how the scheduling heuristic works. The problem consists of 15 jobs to be scheduled on three machines. The height of a machine represents the amount of resources available, while a job's height represents its resource requirement. A colour implies a job's machine, i.e., jobs of one colour must go on one machine (e.g. ``orange'' jobs should be scheduled on machine $m_2$). 

Figure~\ref{subfig:a} shows a permutation $\pi$, and a solution that is partially completed with a few jobs allocated to each machine. The waiting list also consists of some waiting jobs $\hat{\pi}$, where for example, Job $8$ waits for Job $2$ to be scheduled, considering Job $2$ precedes Job $8$. Figure~\ref{subfig:b} shows one step in the method, where the next job in $\pi$, Job 2, is chosen and scheduled in the first available resource feasible time slot on $m_1$. Once Job 2 is scheduled, the waiting list can be examined, and it is found that Job 8 can be released. Job 8 is scheduled as early as possible considering the resource availability. 

\section{Experimental Setting} \label{sec:expts}

We investigate APSA on RCJSU, and compare its performance to current state-of-the-art approaches, including ACS and SACS. Additionally, we investigate APSA against standard simulated annealing (SA) with the same settings, and the SA of Singh and Ernst \cite{singhernst10}. Finally, we examine APSA's performance on the deterministic problem (RCJS), and we compare APSA to ACS, SACS and a column generation and ACO hybrid by  Thiruvady et al. \cite{Thiruvady2014CG}. The experiments are conducted on problem instances from the literature.\footnote{\url{https://github.com/andreas-ernst/Mathprog-ORlib/blob/master/data/RCJS_Instances.zip} consists of problem instances for RCJS. The resource usage within these instances is modified to generate problem instances for RCJSU.} While this data was originally created for the RCJS, each problem instance is modified to generate the data for RCJSU. Specifically, there are 10 samples generated for each problem instance, where each one has resource limits in the range $[U_{min}, U_l \times G]$. $U_{min}$ is chosen to be the largest resource required by any job, i.e. $\max_{i \in \cal J} g_i$, which ensures a resource feasible schedule can always be found. Additionally, the multiplier $U_l \in \{0.5, 0.6, 0.7, 0.8, 0.9, 1.0\}$ is used to select a  proportion of the resources. Typically, a smaller value (e.g. 0.5 or 0.6) means that the problem instance is tightly constrained, where the schedules will need larger completion times, leading to a higher overall TWTs.  

The parameter settings for APSA were selected through tuning by hand and by using past studies as a guide. We refer to Algorithms~\ref{alg:psa} and ~\ref{alg:mh}. The time limit $t_lim$ is set to 600 seconds. The initial temperature $T_0$ was set to 1500 \cite{singhernst10}. Due to uncertainty and the overheads required to construct schedules with a number of samples, the parameters within the Metropolis-Hastings algorithm have been set to allow for much quicker cooling schedule. This is achieved by allowing only 1000 iterations $Iter = 1000$ and $\gamma = 0.5$. The population size was found through tuning by hand $s = 10$. The sub-sequence size used within $\beta$-sampling of length 5 jobs, which is obtained from \cite{singhernst10}. Moreover, in Line~14 of Algorithm~\ref{alg:psa}, there are $m = rand() * n$ applications of job swapping, where $rand()$ is a random number between 0 and 1. The parameter settings for ACS and SACS were obtained from ~\cite{thiruvady2022}, to which we refer the reader for complete details. 

We carry out experiments on all problem instances considering all uncertainty levels. There are 25 repetitions of for each algorithm, with each one given 10 minutes of wall clock time. The memory  limit was 2 GB, which proved to be sufficient for all algorithms. The algorithms were implemented in C++ with GCC-5.4.0. All experiments were carried out on  Monash University's Campus Cluster -- MonARCH.\footnote{\url{https://confluence.apps.monash.edu/display/monarch/MonARCH+Home}}

The IP was implemented in Gurobi 9.0.1 \cite{gurobi12}. It was run for 10 minutes per problem instance and allowed a larger amount of memory of 10GB. Across all the runs, only one solution was found for one problem instance at the uncertainly of 0.8. This is not surprising given the size of the model, which is unsuitable in practical settings.

\subsection{Results} \label{sec:results}

The results that follow are reported as a percentage difference from the average solution quality over 25 runs for each algorithm to the best solution found by any algorithm. Consider a problem instance and let $B*$ be the best solution obtained by any algorithm for that problem instance. Let $A_k$ be the best solution found by APSA for instance $k$. Then the percentage difference for APSA is $\frac{A_k - B*}{A_k}$.  

Tables~\ref{tab:comp1} and ~\ref{tab:comp2} show the results of ACS, SACS and APSA for low and high uncertainty levels $UL$, respectively. The tables show the best result obtained by any algorithm (which is highlighted in bold if it is found by APSA) for each uncertainty level, followed by the percentage differences of each algorithm to the corresponding best solution. The best result found by any algorithm is highlighted in bold, and those results that are statistically significant (determined using a Wilcoxon ranked-sum test) are marked with a ``{\bf *}''. To summarise the performance of the algorithms across the problem instances, the last row of the tables report the number of times ({\bf \# best}) the algorithm finds the best solution. 

\begin{table*}[htb!]
\centering
\caption{Results of ACS, SACS and APSA shown as the percentage difference to the best solution for uncertainty levels 0.5, 0.6 and 0.7; Results in  column ``Best'' are highlighted in bold if APSA finds the best solution; Results for each algorithm are highlighted in bold when they achieve best average results, while statistically significant results (pair-wise Wilcoxon ranked-sum test) at a 95$\%$ confidence interval are marked with a {\bf *}.}
\label{tab:comp1}
 \scalebox{0.7}{
\begin{tabular}{llrrrrrrrrrrrrrrrrr}
\toprule			
	&&			\multicolumn{4}{c}{0.5}					&&			\multicolumn{4}{c}{0.6}					&&			\multicolumn{4}{c}{0.7}					\\
 \midrule
 	&&	Best	&	ACS	&	SACS	&	APSA	&&	Best	&	ACS	&	SACS	&	APSA	&&	Best	&	ACS	&	SACS	&	APSA	\\
\midrule					3-5	&&	{\bf	971.7}	&		1.034		&		1.043		&	{\bf *0.142}	&&	{\bf	942.6}	&		0.494		&		0.404		&	{\bf *0.048}	&&	{\bf	988.8}	&		1.177		&		1.092		&	{\bf *0.260}	\\
3-23	&&	{\bf	224.9}	&		0.934		&		0.791		&	{\bf *0.000}	&&	{\bf	232.5}	&		2.314		&		2.052		&	{\bf *0.138}	&&	{\bf	240.5}	&		2.374		&		2.345		&	{\bf *0.162}	\\
3-53	&&	{\bf	221.5}	&		0.474		&		0.366		&	{\bf *0.000}	&&	{\bf	213.4}	&		0.066		&		0.028		&	{\bf	-0.005}	&&	{\bf	231.7}	&		0.319		&		0.211		&	{\bf *0.000}	\\
4-28	&&	{\bf	150.6}	&		6.754		&		6.768		&	{\bf *0.671}	&&	{\bf	128.1}	&		6.409		&		6.300		&	{\bf *1.046}	&&	{\bf	116.1}	&		6.763		&		6.720		&	{\bf *0.164}	\\
4-42	&&	{\bf	401.7}	&		1.327		&		1.279		&	{\bf *0.144}	&&	{\bf	361.1}	&		1.603		&		1.606		&	{\bf *0.219}	&&	{\bf	323.9}	&		0.540		&		0.509		&	{\bf *0.034}	\\
4-61	&&	{\bf	248.3}	&		5.578		&		5.324		&	{\bf *0.209}	&&	{\bf	213.2}	&		4.877		&		4.192		&	{\bf *0.225}	&&	{\bf	196.6}	&		8.246		&		7.519		&	{\bf *1.073}	\\
5-7	&&	{\bf	801.3}	&		3.582		&		3.590		&	{\bf *0.604}	&&	{\bf	673.3}	&		4.496		&		4.147		&	{\bf *0.371}	&&	{\bf	656.6}	&		3.925		&		3.932		&	{\bf *0.801}	\\
5-21	&&	{\bf	683.6}	&		4.640		&		4.353		&	{\bf *1.066}	&&	{\bf	535.8}	&		6.148		&		5.088		&	{\bf *0.649}	&&	{\bf	577.4}	&		5.504		&		4.342		&	{\bf *1.320}	\\
5-62	&&	{\bf	816.7}	&		4.660		&		4.450		&	{\bf *0.579}	&&	{\bf	688.0}	&		5.712		&		5.436		&	{\bf *1.145}	&&	{\bf	676.5}	&		5.117		&		5.175		&	{\bf *1.042}	\\
\midrule
6-10	&&	{\bf	1971.2}	&		4.115		&		3.230		&	{\bf *0.859}	&&	{\bf	2026.3}	&		4.708		&		3.944		&	{\bf *1.610}	&&	{\bf	1583.3}	&		5.018		&		4.215		&	{\bf *1.520}	\\
6-28	&&	{\bf	739.3}	&		5.201		&		5.212		&	{\bf *0.587}	&&	{\bf	762.4}	&		5.105		&		4.963		&	{\bf *1.229}	&&	{\bf	574.8}	&		6.738		&		6.692		&	{\bf *1.541}	\\
6-58	&&	{\bf	763.9}	&		4.778		&		4.590		&	{\bf *1.468}	&&	{\bf	799.0}	&		4.781		&		4.303		&	{\bf *1.055}	&&	{\bf	587.2}	&		6.391		&		6.214		&	{\bf *0.823}	\\
7-5	&&	{\bf	1235.3}	&		2.203		&		2.054		&	{\bf *0.131}	&&	{\bf	1025.5}	&		3.745		&		3.631		&	{\bf *0.725}	&&	{\bf	1213.8}	&		2.963		&		2.833		&	{\bf *0.386}	\\
7-23	&&	{\bf	1426.7}	&		6.619		&		5.917		&	{\bf *1.641}	&&	{\bf	1194.3}	&		6.363		&		6.007		&	{\bf *1.776}	&&	{\bf	1391.4}	&		6.034		&		5.365		&	{\bf *1.286}	\\
7-47	&&	{\bf	1442.1}	&		6.778		&		5.901		&	{\bf *1.540}	&&	{\bf	1191.5}	&		6.496		&		5.815		&	{\bf *1.461}	&&	{\bf	1410.3}	&		6.524		&		5.664		&	{\bf *1.529}	\\
8-3	&&	{\bf	1856.4}	&		10.401		&		7.545		&	{\bf *1.670}	&&	{\bf	1995.5}	&		10.389		&		8.257		&	{\bf *2.845}	&&	{\bf	1640.1}	&		11.968		&		9.434		&	{\bf *3.054}	\\
8-53	&&	{\bf	1154.1}	&		6.363		&		6.101		&	{\bf *1.623}	&&	{\bf	1250.5}	&		5.754		&		5.576		&	{\bf *1.181}	&&	{\bf	1053.8}	&		6.512		&		6.794		&	{\bf *1.839}	\\
8-77	&&	{\bf	2799.0}	&		5.864		&		4.565		&	{\bf *1.612}	&&	{\bf	2971.5}	&		5.134		&		4.193		&	{\bf *1.135}	&&	{\bf	2570.1}	&		5.883		&		5.522		&	{\bf *2.000}	\\
\midrule
9-20	&&	{\bf	2462.6}	&		5.900		&		5.301		&	{\bf *1.648}	&&	{\bf	2086.1}	&		5.375		&		5.114		&	{\bf *1.227}	&&	{\bf	2125.4}	&		5.644		&		5.402		&	{\bf *1.836}	\\
9-47	&&	{\bf	3185.0}	&		6.684		&		5.270		&	{\bf *2.029}	&&	{\bf	2640.4}	&		7.618		&		5.934		&	{\bf *1.993}	&&	{\bf	3333.7}	&		6.529		&		5.134		&	{\bf *2.123}	\\
9-62	&&	{\bf	3485.6}	&		4.678		&		4.103		&	{\bf *1.154}	&&	{\bf	2969.1}	&		4.414		&		3.864		&	{\bf *0.869}	&&	{\bf	3656.7}	&		4.199		&		3.589		&	{\bf *1.060}	\\
10-7	&&	{\bf	5854.2}	&		4.066		&		3.749		&	{\bf *0.875}	&&	{\bf	5603.8}	&		4.764		&		4.108		&	{\bf *1.387}	&&	{\bf	5126.2}	&		4.792		&		4.332		&	{\bf *1.198}	\\
10-13	&&	{\bf	5246.0}	&		3.092		&		2.783		&	{\bf *0.769}	&&	{\bf	4932.8}	&		4.675		&		4.018		&	{\bf *1.607}	&&	{\bf	4517.0}	&		4.719		&		4.497		&	{\bf *1.694}	\\
10-31	&&	{\bf	1603.1}	&		6.196		&		6.914		&	{\bf *1.724}	&&	{\bf	1525.3}	&		6.018		&		6.445		&	{\bf *1.210}	&&	{\bf	1369.6}	&		6.959		&		7.841		&	{\bf *1.608}	\\
11-21	&&	{\bf	2359.0}	&		4.747		&		5.053		&	{\bf *1.167}	&&	{\bf	2402.8}	&		4.731		&		4.919		&	{\bf *0.990}	&&	{\bf	2502.8}	&		5.099		&		4.934		&	{\bf *1.247}	\\
11-56	&&	{\bf	3991.0}	&		4.836		&		4.420		&	{\bf *1.185}	&&	{\bf	4047.3}	&		5.072		&		4.196		&	{\bf *1.177}	&&	{\bf	4156.7}	&		5.714		&		5.204		&	{\bf *2.094}	\\
11-63	&&	{\bf	4354.2}	&		3.885		&		3.518		&	{\bf *0.863}	&&	{\bf	4196.9}	&		5.211		&		4.887		&	{\bf *1.682}	&&	{\bf	4561.3}	&		3.580		&		3.516		&	{\bf *0.636}	\\
\midrule
12-14	&&	{\bf	4405.2}	&		4.136		&		4.295		&	{\bf *1.429}	&&	{\bf	4378.3}	&		4.937		&		4.890		&	{\bf *1.944}	&&	{\bf	3636.3}	&		4.592		&		4.533		&	{\bf *1.352}	\\
12-36	&&	{\bf	6825.5}	&		2.979		&		3.217		&	{\bf *1.484}	&&	{\bf	6844.5}	&		3.762		&		3.376		&	{\bf *1.435}	&&	{\bf	5710.1}	&		3.732		&		3.504		&	{\bf *1.429}	\\
12-80	&&	{\bf	5234.4}	&		6.105		&		5.769		&	{\bf *2.985}	&&	{\bf	5310.1}	&		4.949		&		4.875		&	{\bf *2.004}	&&	{\bf	4460.8}	&		5.080		&		4.653		&	{\bf *2.376}	\\
15-2	&&	{\bf	8105.5}	&		3.879		&		4.319		&	{\bf *3.304}	&&	{\bf	7577.1}	&		3.058		&		3.153		&	{\bf *2.565}	&&	{\bf	7737.2}	&		2.834		&		2.843		&	{\bf *2.015}	\\
15-3	&&	{\bf	9206.1}	&		4.375		&		4.181		&	{\bf *2.511}	&&	{\bf	8585.2}	&		3.669		&		3.626		&	{\bf *2.008}	&&	{\bf	8820.9}	&		2.820		&		2.663		&	{\bf *0.972}	\\
15-5	&&	{\bf	7548.8}	&		5.099		&		5.345		&	{\bf *2.547}	&&	{\bf	7021.0}	&		4.965		&		4.958		&	{\bf *2.430}	&&	{\bf	7186.8}	&		3.964		&		3.968		&	{\bf *1.876}	\\
20-2	&&	{\bf	15042.1}	&		2.994		&		3.320		&	{\bf *1.635}	&&	{\bf	14021.9}	&		2.747		&		2.957		&	{\bf *1.928}	&&	{\bf	15724.6}	&		2.889		&		2.738		&	{\bf *1.793}	\\
20-5	&&	{\bf	25172.0}	&		3.239		&		2.917		&	{\bf *1.693}	&&	{\bf	23175.3}	&		3.832		&		4.185		&	{\bf *2.425}	&&	{\bf	26028.0}	&		3.782		&		4.010		&	{\bf *2.558}	\\
20-6	&&	{\bf	13733.3}	&		2.596		&		2.678		&	{\bf *1.614}	&&	{\bf	12713.5}	&		2.707		&		2.725		&	{\bf *1.900}	&&	{\bf	14329.1}	&		2.554		&		2.630		&	{\bf *1.713}	\\

\midrule
{\bf $\#$ best} && \multicolumn{1}{c}{{\bf 36}} & \multicolumn{1}{c}{{\bf 0}} & \multicolumn{1}{c}{{\bf 0}} & \multicolumn{1}{c}{{\bf 36}} && \multicolumn{1}{c}{{\bf 36}} & \multicolumn{1}{c}{{\bf 0}} & \multicolumn{1}{c}{{\bf 0}} & \multicolumn{1}{c}{{\bf 36}} && \multicolumn{1}{c}{{\bf 36}}  &  \multicolumn{1}{c}{{\bf 0}} & \multicolumn{1}{c}{{\bf 0}} & \multicolumn{1}{c}{{\bf 36}} \\
\bottomrule
\end{tabular}
}
\end{table*}

The results in both tables show that APSA finds best known solutions for all problem instances and uncertainty levels (except instance 15--2 at uncertainty level 0.9), and the best average results across all problem instances for all uncertainty levels. Looking closely, we see that the differences between the algorithms are smallest for the smallest problems (3 machines) and largest problems (20 machines). Everywhere else, APSA has significant advantages. For the problem instances with 3 machines, this is not surprising, because it is expected that all algorithms find solutions that are close to optimal. However, APSA has the advantage that, for every run, it finds a solution close to the best. On investigating the individual runs on the large problem instances, we find that there are relatively few iterations conducted ($<$ 10). This leads to APSA not converging sufficiently well, thereby not finding as large improvements are seen on other smaller problem instances.  

\begin{table*}[htb!]
\centering
\caption{Results of ACS, SACS and APSA shown as the percentage difference to the best solution for uncertainty levels 0.8, 0.9 and 1.0; Results in  column ``Best'' are highlighted in bold if APSA finds the best solution; Results for each algorithm are highlighted in bold when they achieve best average results, while statistically significant results (pair-wise Wilcoxon ranked-sum test) at a 95$\%$ confidence interval are marked with a {\bf *}.} 
\label{tab:comp2}
 \scalebox{0.7}{
\begin{tabular}{llrrrrrrrrrrrrrrrrr}
\toprule			
	&&			\multicolumn{4}{c}{0.8}					&&			\multicolumn{4}{c}{0.9}					&&			\multicolumn{4}{c}{1.0}					\\
 \midrule
  	&&	Best	&	ACS	&	SACS	&	APSA	&&	Best	&	ACS	&	SACS	&	APSA	&&	Best	&	ACS	&	SACS	&	APSA	\\
\midrule					
3-5	&&	{\bf	861.9}	&		0.972		&		0.773		&	{\bf *0.031}	&&	{\bf	949.2}	&		0.165		&		0.162		&	{\bf *0.000}	&&	{\bf	786.3}	&		0.424		&		0.439		&	{\bf *0.050}	\\
3-23	&&	{\bf	214.6}	&		1.981		&		1.631		&	{\bf *0.084}	&&	{\bf	232.2}	&		1.636		&		1.666		&	{\bf *0.000}	&&	{\bf	200.9}	&		2.787		&		2.568		&	{\bf *0.244}	\\
3-53	&&	{\bf	186.0}	&		0.016		&		0.016		&	{\bf *0.000}	&&	{\bf	212.5}	&		0.085		&		0.047		&	{\bf *0.000}	&&	{\bf	145.5}	&		0.069		&		0.041		&	{\bf *0.000}	\\
4-28	&&	{\bf	101.5}	&		7.224		&		6.800		&	{\bf *0.857}	&&	{\bf	113.3}	&		4.556		&		4.830		&	{\bf *0.618}	&&	{\bf	64.5}	&		8.632		&		8.368		&	{\bf *1.364}	\\
4-42	&&	{\bf	279.3}	&		1.891		&		1.457		&	{\bf *0.329}	&&	{\bf	300.6}	&		0.935		&		0.975		&	{\bf *0.116}	&&	{\bf	212.4}	&		2.241		&		1.973		&	{\bf *0.009}	\\
4-61	&&	{\bf	168.7}	&		7.753		&		7.327		&	{\bf *0.854}	&&	{\bf	186.6}	&		6.062		&		5.729		&	{\bf *1.388}	&&	{\bf	114.9}	&		6.922		&		6.983		&	{\bf *1.654}	\\
5-7	&&	{\bf	573.2}	&		4.079		&		3.704		&	{\bf *0.406}	&&	{\bf	632.9}	&		4.775		&		4.399		&	{\bf *0.242}	&&	{\bf	589.9}	&		4.582		&		3.841		&	{\bf *0.071}	\\
5-21	&&	{\bf	445.1}	&		5.687		&		5.028		&	{\bf *1.654}	&&	{\bf	488.2}	&		5.930		&		4.931		&	{\bf *1.540}	&&	{\bf	505.9}	&		5.582		&		4.556		&	{\bf *1.044}	\\
5-62	&&	{\bf	575.7}	&		6.429		&		5.659		&	{\bf *1.553}	&&	{\bf	645.9}	&		6.114		&		5.526		&	{\bf *1.118}	&&	{\bf	592.3}	&		5.185		&		4.788		&	{\bf *1.047}	\\
\midrule
6-10	&&	{\bf	1880.3}	&		4.587		&		3.369		&	{\bf *0.899}	&&	{\bf	1404.6}	&		5.327		&		4.447		&	{\bf *1.101}	&&	{\bf	1674.6}	&		3.967		&		3.259		&	{\bf *0.809}	\\
6-28	&&	{\bf	702.7}	&		6.215		&		5.974		&	{\bf *1.083}	&&	{\bf	502.2}	&		6.283		&		6.263		&	{\bf *1.296}	&&	{\bf	598.6}	&		5.578		&		5.366		&	{\bf *0.984}	\\
6-58	&&	{\bf	736.0}	&		4.356		&		3.806		&	{\bf *1.067}	&&	{\bf	517.5}	&		6.059		&		5.606		&	{\bf *1.015}	&&	{\bf	794.9}	&		4.651		&		3.495		&	{\bf *0.742}	\\
7-5	&&	{\bf	992.0}	&		3.793		&		3.876		&	{\bf *0.812}	&&	{\bf	1277.6}	&		2.216		&		2.062		&	{\bf *0.272}	&&	{\bf	919.6}	&		4.071		&		3.771		&	{\bf *1.075}	\\
7-23	&&	{\bf	1163.8}	&		6.092		&		5.997		&	{\bf *1.662}	&&	{\bf	1466.6}	&		5.974		&		5.281		&	{\bf *1.206}	&&	{\bf	1058.8}	&		7.126		&		6.215		&	{\bf *1.327}	\\
7-47	&&	{\bf	1135.7}	&		7.691		&		6.703		&	{\bf *2.246}	&&	{\bf	1493.3}	&		6.067		&		5.524		&	{\bf *1.051}	&&	{\bf	1031.0}	&		8.798		&		7.950		&	{\bf *2.866}	\\
8-3	&&	{\bf	1594.1}	&		10.686		&		8.116		&	{\bf *2.517}	&&	{\bf	1740.7}	&		9.527		&		7.545		&	{\bf *1.922}	&&	{\bf	1077.0}	&		13.622		&		10.734		&	{\bf *2.893}	\\
8-53	&&	{\bf	1041.5}	&		6.370		&		5.954		&	{\bf *1.625}	&&	{\bf	1107.4}	&		5.501		&		5.821		&	{\bf *1.518}	&&	{\bf	752.1}	&		7.790		&		7.627		&	{\bf *2.153}	\\
8-77	&&	{\bf	2499.8}	&		5.215		&		4.627		&	{\bf *1.302}	&&	{\bf	2669.8}	&		4.973		&		4.097		&	{\bf *1.163}	&&	{\bf	1861.2}	&		6.300		&		4.572		&	{\bf *1.253}	\\
\midrule
9-20	&&	{\bf	2277.7}	&		5.016		&		4.463		&	{\bf *1.343}	&&	{\bf	1964.0}	&		5.258		&		5.149		&	{\bf *1.170}	&&	{\bf	1959.1}	&		6.745		&		6.119		&	{\bf *2.360}	\\
9-47	&&	{\bf	2868.8}	&		6.328		&		4.602		&	{\bf *1.956}	&&	{\bf	2048.5}	&		6.802		&		5.562		&	{\bf *1.512}	&&	{\bf	2312.4}	&		6.872		&		5.587		&	{\bf *1.624}	\\
9-62	&&	{\bf	3189.2}	&		4.428		&		4.091		&	{\bf *1.275}	&&	{\bf	2316.4}	&		5.602		&		4.744		&	{\bf *1.526}	&&	{\bf	2609.0}	&		4.345		&		3.926		&	{\bf *0.947}	\\
10-7	&&	{\bf	4216.6}	&		5.468		&		5.133		&	{\bf *2.104}	&&	{\bf	4675.7}	&		4.641		&		4.449		&	{\bf *1.545}	&&	{\bf	4010.0}	&		5.551		&		5.095		&	{\bf *1.859}	\\
10-13	&&	{\bf	3721.5}	&		4.347		&		4.032		&	{\bf *1.152}	&&	{\bf	4107.9}	&		4.741		&		4.474		&	{\bf *1.548}	&&	{\bf	3522.7}	&		4.906		&		4.490		&	{\bf *1.417}	\\
10-31	&&	{\bf	1100.6}	&		8.086		&		8.739		&	{\bf *1.873}	&&	{\bf	1249.2}	&		5.754		&		6.423		&	{\bf *1.094}	&&	{\bf	1041.7}	&		6.940		&		7.985		&	{\bf *1.382}	\\
11-21	&&	{\bf	1948.8}	&		5.689		&		5.640		&	{\bf *1.829}	&&	{\bf	1940.4}	&		5.596		&		5.714		&	{\bf *1.757}	&&	{\bf	2050.3}	&		5.098		&		4.992		&	{\bf *1.648}	\\
11-56	&&	{\bf	3322.7}	&		5.332		&		4.479		&	{\bf *1.307}	&&	{\bf	3291.5}	&		5.616		&		5.364		&	{\bf *1.741}	&&	{\bf	3452.8}	&		4.922		&		4.790		&	{\bf *1.473}	\\
11-63	&&	{\bf	3504.1}	&		4.215		&		4.112		&	{\bf *1.296}	&&	{\bf	3312.6}	&		4.210		&		4.106		&	{\bf *1.008}	&&	{\bf	3670.2}	&		3.747		&		3.788		&	{\bf *0.900}	\\
\midrule
12-14	&&	{\bf	3625.5}	&		4.608		&		4.575		&	{\bf *1.365}	&&	{\bf	2998.5}	&		4.821		&		5.534		&	{\bf *1.807}	&&	{\bf	2781.7}	&		6.100		&		5.753		&	{\bf *2.273}	\\
12-36	&&	{\bf	5705.1}	&		3.695		&		3.216		&	{\bf *1.549}	&&	{\bf	4798.6}	&		4.066		&		3.667		&	{\bf *1.820}	&&	{\bf	5923.6}	&		3.240		&		3.124		&	{\bf *1.236}	\\
12-80	&&	{\bf	4447.2}	&		5.204		&		4.647		&	{\bf *2.377}	&&	{\bf	3754.3}	&		5.471		&		5.201		&	{\bf *2.299}	&&	{\bf	3534.2}	&		5.557		&		5.138		&	{\bf *2.620}	\\
15-2	&&	{\bf	6311.7}	&		2.411		&		2.192		&	{\bf *1.555}	&&	7051.6	&		3.569		&		3.568		&	{\bf *3.032}	&&	{\bf	5905.5}	&		2.888		&		2.977		&	{\bf *2.104}	\\
15-3	&&	{\bf	7074.2}	&		3.101		&		3.040		&	{\bf *1.102}	&&	{\bf	8037.3}	&		2.679		&		2.814		&	{\bf *1.449}	&&	{\bf	6630.9}	&		3.497		&		3.397		&	{\bf *1.511}	\\
15-5	&&	{\bf	5753.9}	&		4.440		&		4.115		&	{\bf *1.766}	&&	{\bf	6473.3}	&		6.295		&		5.637		&	{\bf *3.051}	&&	{\bf	5431.4}	&		4.227		&		4.047		&	{\bf *1.682}	\\
20-2	&&	{\bf	13838.7}	&		3.476		&		3.375		&	{\bf *2.390}	&&	{\bf	13081.1}	&		2.433		&		2.402		&	{\bf *1.347}	&&	{\bf	12956.3}	&		3.272		&		3.372		&	{\bf *2.110}	\\
20-5	&&	{\bf	22944.5}	&		4.050		&		3.861		&	{\bf *2.593}	&&	{\bf	21584.8}	&		3.658		&		3.912		&	{\bf *2.586}	&&	{\bf	21684.4}	&		2.997		&		2.975		&	{\bf *1.512}	\\
20-6	&&	{\bf	12518.9}	&		3.399		&		3.142		&	{\bf *2.205}	&&	{\bf	11895.5}	&		1.884		&		2.172		&	{\bf *1.169}	&&	{\bf	11880.4}	&		1.851		&		1.780		&	{\bf *1.014}	\\

\midrule
{\bf $\#$ best} && \multicolumn{1}{c}{{\bf 36}} & \multicolumn{1}{c}{{\bf 0}} & \multicolumn{1}{c}{{\bf 0}} & \multicolumn{1}{c}{{\bf 36}} && \multicolumn{1}{c}{{\bf 35}} & \multicolumn{1}{c}{{\bf 0}} & \multicolumn{1}{c}{{\bf 0}} & \multicolumn{1}{c}{{\bf 36}} && \multicolumn{1}{c}{{\bf 36}} &  \multicolumn{1}{c}{{\bf 0}} & \multicolumn{1}{c}{{\bf 0}} & \multicolumn{1}{c}{{\bf 36}} \\

\bottomrule
\end{tabular}
}
\end{table*}

The results in Tables~\ref{tab:comp1} and ~\ref{tab:comp2} are summarised in Figure~\ref{fig:acs-sacs-conv}. The figure shows the performance of each algorithm (ACS, SACS and APSA) by uncertainty levels. The bars represent the average of percentage difference to best across all instances in an uncertainty level. As expected, APSA clearly outperforms the other two methods, with a slight advantage with increasing uncertainty levels.

\begin{figure}[t]
\centering
\includegraphics[clip,width=0.5\textwidth]{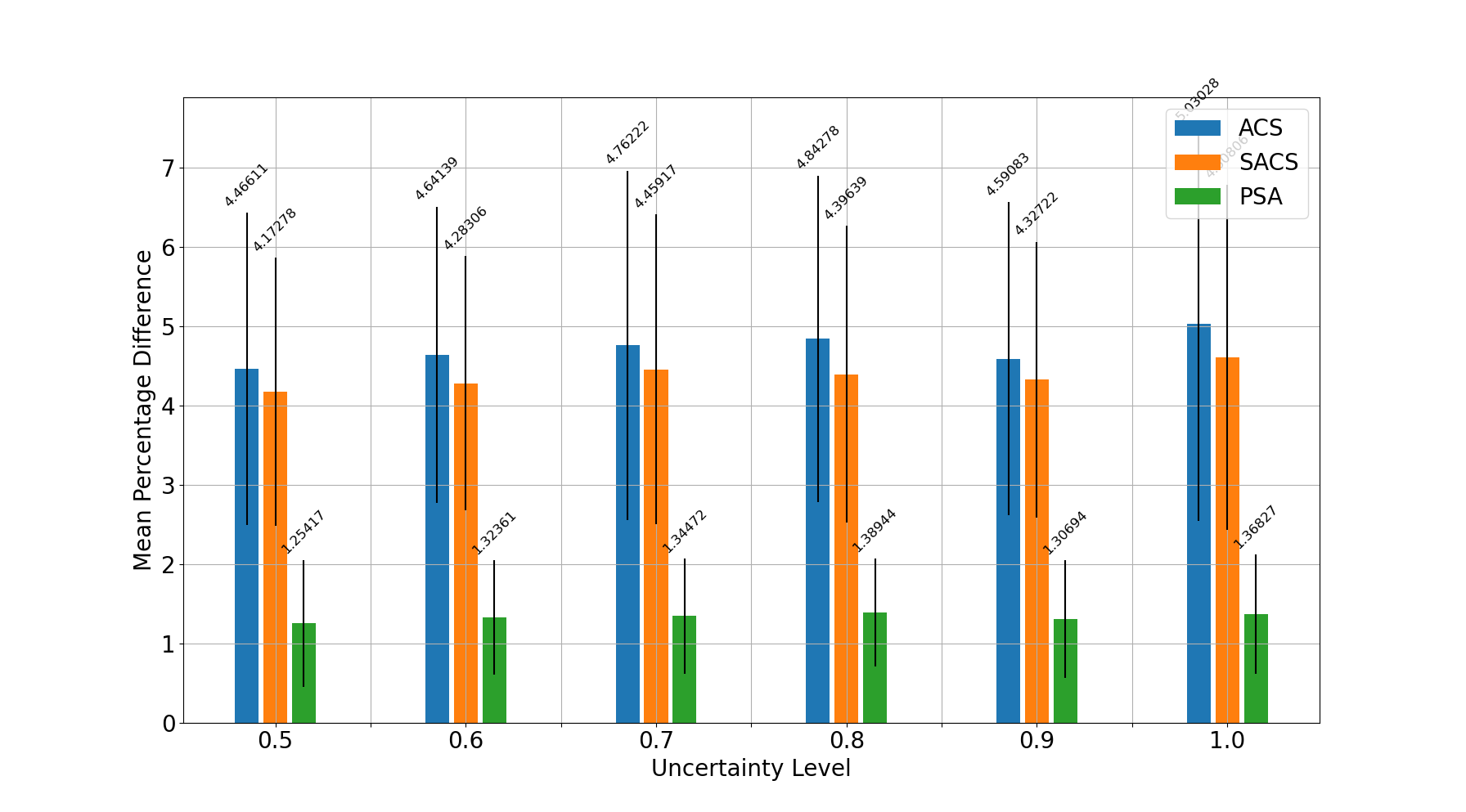}
\caption{Comparison of ACS, SACS and APSA by uncertainty levels.} \label{fig:apsa_vs_acs}
\end{figure}

\subsection{Convergence Characteristics of APSA} \label{sec:conv_char}

APSA outperforms SACS and ACS across all problem instances for all uncertainty levels, due to its ability to diversify well. In particular, APSA's ability to stop getting stuck in local optima leads to substantial gains.

In the following, we examine the convergence of APSA, and compare it to ACS and SACS. We choose four problem instances, which are sufficiently diverse in size to allow making somewhat general inferences about APSA. The results are shown in Figure~\ref{fig:acs-sacs-conv}, where each sub-figure corresponds to one problem instance. The plots show time in seconds on the x-axis and TWT (logarithmic scale) on the y-axis.

\begin{figure*}[htbp]
\centering
\includegraphics[clip,width=17cm]{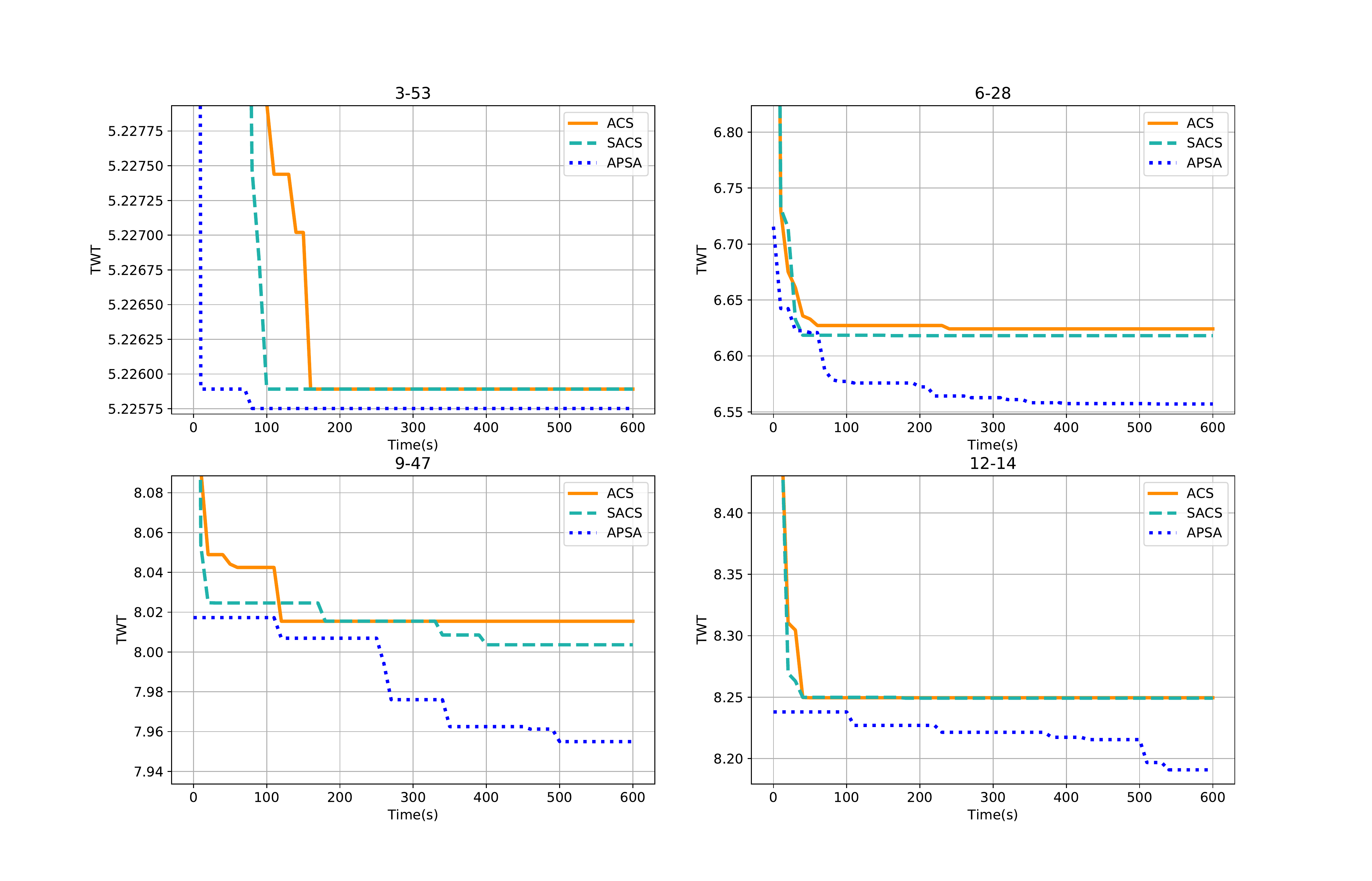}
\caption{Convergence profiles of ACS, SACS and APSA on four RCJSU problem instances of different sizes. The instances chosen are 3--53, 6--28, 9--47 and 12--14. The TWTs on the y-axes are presented on a logarithmic scale.} \label{fig:acs-sacs-conv}
\end{figure*}

Again, APSA is clearly the best performing method across all four problem instances. There are two main advantages of the algorithm, especially related to the population and diversification. Firstly, the starting solution, irrespective of the problem instance, is better than ACS or SACS. This can be attributed to the population, where at the beginning of the search process, there is a substantial amount of diversity in the starting solutions. Secondly, and more importantly, there is continuous improvement in APSA over-time. This can be clearly seen for medium to large problem instances (6--28, 9--47 and 12--14). The increased diversity allowed by the re-initialisation of the solutions in the population ensures that the population does not get stuck in local optima. 

\subsection{Comparisons of APSA to the original SA (\cite{singhernst10})}\label{sec:population}

The original SA (OSA) algorithm  \cite{singhernst10} can be adapted to tackle the RCJSU problem.\footnote{Full details of the algorithm can be found in \cite{singhernst10}, and due to space limitations, we do not provide the full algorithm or its description here.} The main difference is that every time a new solution is created, a new objective needs to be computed, which involves creating multiple schedules corresponding to the samples. Otherwise, all the settings for the algorithm remain the same.

\begin{table}[htb!]
\centering
\caption{A comparison of PSA and the original SA  \cite{singhernst10} (SAO) for the uncertainty levels of 0.6 and 0.9.} 
\label{tab:PSAVsOrig}
 \scalebox{0.8}{
\begin{tabular}{llrrrrrrr}
\toprule			
	&&			\multicolumn{3}{c}{0.6}					&&			\multicolumn{3}{c}{0.9}										\\
\midrule																									
	&&	Best	&	APSA	&	SAO	&&	Best	&	APSA	&	SAO \\
\midrule	
3-5	&&	942.64	&	0.05	&	0.03	&&	949.18	&	0.00	&	0.00	\\
3-23	&&	232.48	&	0.14	&	0.62	&&	232.24	&	0.00	&	0.09	\\
3-53	&&	213.37	&	0.00	&	0.00	&&	212.45	&	0.00	&	0.00	\\
4-28	&&	127.92	&	1.19	&	2.19	&&	113.26	&	0.62	&	1.28	\\
4-42	&&	361.12	&	0.22	&	0.47	&&	300.61	&	0.12	&	0.10	\\
4-61	&&	213.24	&	0.23	&	1.02	&&	186.58	&	1.39	&	2.69	\\
5-7	&&	673.31	&	0.37	&	1.15	&&	632.89	&	0.24	&	0.68	\\
5-21	&&	535.80	&	0.65	&	2.87	&&	488.18	&	1.54	&	3.14	\\
5-62	&&	688.01	&	1.15	&	3.25	&&	645.89	&	1.12	&	3.79	\\
\midrule
6-10	&&	2026.29	&	1.61	&	2.65	&&	1404.63	&	1.10	&	2.96	\\
6-28	&&	762.44	&	1.23	&	2.33	&&	502.18	&	1.30	&	2.53	\\
6-58	&&	798.97	&	1.06	&	2.30	&&	517.45	&	1.01	&	3.37	\\
7-5	&&	1025.49	&	0.72	&	1.82	&&	1277.61	&	0.27	&	0.81	\\
7-23	&&	1194.27	&	1.78	&	4.06	&&	1466.58	&	1.21	&	2.60	\\
7-47	&&	1191.50	&	1.46	&	4.30	&&	1493.28	&	1.05	&	4.49	\\
8-3	&&	1995.45	&	2.84	&	6.84	&&	1740.71	&	1.92	&	5.11	\\
8-53	&&	1250.50	&	1.18	&	3.11	&&	1107.35	&	1.52	&	3.26	\\
8-77	&&	2971.52	&	1.13	&	2.74	&&	2669.8	&	1.16	&	2.52	\\
\midrule
9-20	&&	2086.08	&	1.23	&	3.23	&&	1964.04	&	1.17	&	3.13	\\
9-47	&&	2640.36	&	1.99	&	3.98	&&	2048.54	&	1.51	&	3.95	\\
9-62	&&	2969.09	&	0.87	&	2.42	&&	2316.43	&	1.53	&	3.47	\\
10-7	&&	5602.57	&	1.41	&	2.59	&&	4675.71	&	1.55	&	3.14	\\
10-13	&&	4932.81	&	1.61	&	2.56	&&	4107.91	&	1.55	&	3.15	\\
10-31	&&	1525.30	&	1.21	&	2.42	&&	1249.17	&	1.09	&	3.05	\\
11-21	&&	2402.84	&	0.99	&	2.69	&&	1940.39	&	1.76	&	3.46	\\
11-56	&&	4046.64	&	1.19	&	2.54	&&	3291.52	&	1.74	&	3.09	\\
11-63	&&	4196.87	&	1.68	&	2.69	&&	3312.59	&	1.01	&	1.94	\\
\midrule
12-14	&&	4378.29	&	1.94	&	2.56	&&	2998.54	&	1.81	&	3.21	\\
12-36	&&	6821.22	&	1.78	&	2.23	&&	4798.61	&	1.82	&	2.52	\\
12-80	&&	5310.08	&	2.00	&	2.11	&&	3754.28	&	2.30	&	3.13	\\
15-2	&&	7554.56	&	2.87	&	2.27	&&	7144.34	&	1.69	&	2.04	\\
15-3	&&	8585.22	&	2.01	&	2.20	&&	8037.26	&	1.45	&	1.41	\\
15-5	&&	7021.04	&	2.43	&	2.71	&&	6473.33	&	3.05	&	3.68	\\
20-2	&&	13813.10	&	3.47	&	3.42	&&	13081.1	&	1.35	&	2.21	\\
20-5	&&	23175.30	&	2.42	&	3.53	&&	21584.8	&	2.59	&	2.98	\\
20-6	&&	12713.50	&	1.90	&	2.13	&&	11895.5	&	1.17	&	1.03	\\

\midrule
{\bf $\#$ best} && & \multicolumn{1}{c}{{\bf 32}} & \multicolumn{1}{c}{{\bf 2}} && & \multicolumn{1}{c}{{\bf 31}} & \multicolumn{1}{c}{{\bf 3}} \\

\bottomrule
\end{tabular}
}
\end{table}

Table~\ref{tab:PSAVsOrig} shows the results for APSA and OSA considering the uncertainty levels 0.6 and 0.9. Other than the smallest problem instances (3 machines), APSA clearly is superior, finding better average solutions ($>$ 85$\%$) across both uncertainty levels.  

\subsection{Comparison of APSA with SA}\label{sec:population}

The novel components of APSA including the population of solutions, adaptive probabilistic selection and quicker cooling schedule clearly lead to significant gains over traditional OSA. In this section, we aim to understand the impacts of the population on APSA. This essentially involves an implementation of APSA with a population size of 1. 

\begin{table}[htb!]
\centering
\caption{A comparison of PSA and SA for the uncertainty levels of 0.6 and 0.9.} 
\label{tab:PSAVsSA}
 \scalebox{0.8}{
\begin{tabular}{llrrrrrrr}
\toprule			
	&&			\multicolumn{3}{c}{0.6}					&&			\multicolumn{3}{c}{0.9}										\\
\midrule																									
	&&	Best	&	PSA	&	SA	&&	Best	&	PSA	&	SA	\\
\midrule	

3-5	&&	942.64	&	0.05	&	0.02	&&	949.18	&	0.00	&	0.00	\\
3-23	&&	232.48	&	0.14	&	0.19	&&	232.24	&	0.00	&	0.00	\\
3-53	&&	213.37	&	0.00	&	0.00	&&	212.45	&	0.00	&	0.00	\\
4-28	&&	127.92	&	1.19	&	0.85	&&	113.26	&	0.62	&	0.68	\\
4-42	&&	361.12	&	0.22	&	0.14	&&	300.61	&	0.12	&	0.10	\\
4-61	&&	213.24	&	0.23	&	0.24	&&	186.15	&	1.62	&	1.58	\\
5-7	&&	673.31	&	0.37	&	0.40	&&	632.89	&	0.24	&	0.27	\\
5-21	&&	534.38	&	0.92	&	1.31	&&	487.22	&	1.74	&	1.70	\\
5-62	&&	688.01	&	1.15	&	1.53	&&	645.89	&	1.12	&	1.69	\\
\midrule
6-10	&&	2026.29	&	1.61	&	1.52	&&	1401.94	&	1.29	&	1.35	\\
6-28	&&	762.44	&	1.23	&	1.28	&&	501.61	&	1.41	&	1.24	\\
6-58	&&	798.94	&	1.06	&	1.23	&&	514.92	&	1.51	&	1.44	\\
7-5	&&	1025.15	&	0.76	&	0.93	&&	1277.4	&	0.29	&	0.29	\\
7-23	&&	1194.27	&	1.78	&	2.08	&&	1463.78	&	1.40	&	1.62	\\
7-47	&&	1184.72	&	2.04	&	2.21	&&	1482.06	&	1.82	&	2.62	\\
8-3	&&	1995.45	&	2.84	&	3.88	&&	1740.22	&	1.95	&	2.12	\\
8-53	&&	1245.14	&	1.62	&	1.81	&&	1107.35	&	1.52	&	1.70	\\
8-77	&&	2971.52	&	1.13	&	1.14	&&	2669.8	&	1.16	&	1.48	\\
\midrule
9-20	&&	2081.38	&	1.46	&	1.58	&&	1957.18	&	1.52	&	1.67	\\
9-47	&&	2634.12	&	2.23	&	2.57	&&	2042.29	&	1.82	&	2.92	\\
9-62	&&	2960.89	&	1.15	&	1.37	&&	2316.43	&	1.53	&	1.72	\\
10-7	&&	5578.66	&	1.84	&	1.74	&&	4675.71	&	1.55	&	2.03	\\
10-13	&&	4911.50	&	2.05	&	1.89	&&	4107.3	&	1.56	&	2.06	\\
10-31	&&	1516.59	&	1.79	&	1.35	&&	1249.17	&	1.09	&	1.36	\\
11-21	&&	2370.51	&	2.37	&	2.28	&&	1940.39	&	1.76	&	2.01	\\
11-56	&&	4010.97	&	2.09	&	2.01	&&	3291.52	&	1.74	&	2.09	\\
11-63	&&	4196.87	&	1.68	&	1.48	&&	3300.36	&	1.38	&	1.23	\\
\midrule
12-14	&&	4375.23	&	2.01	&	1.69	&&	2998.54	&	1.81	&	2.01	\\
12-36	&&	6782.42	&	2.36	&	2.27	&&	4777.3	&	2.27	&	2.25	\\
12-80	&&	5224.42	&	3.68	&	3.13	&&	3719.45	&	3.26	&	3.75	\\
15-2	&&	7538.83	&	3.09	&	3.39	&&	7038.4	&	3.23	&	2.90	\\
15-3	&&	8514.53	&	2.85	&	2.48	&&	7929.08	&	2.83	&	2.43	\\
15-5	&&	7020.24	&	2.44	&	2.25	&&	6473.33	&	3.05	&	3.09	\\
20-2	&&	13893.20	&	2.87	&	3.46	&&	13081.1	&	1.35	&	1.64	\\
20-5	&&	23173.90	&	2.43	&	3.50	&&	21584.8	&	2.59	&	2.81	\\
20-6	&&	12617.80	&	2.67	&	2.93	&&	11707.5	&	2.79	&	2.84	\\

\midrule
{\bf $\#$ best} && & \multicolumn{1}{c}{{\bf 21}} & \multicolumn{1}{c}{{\bf 15}} && & \multicolumn{1}{c}{{\bf 27}} & \multicolumn{1}{c}{{\bf 9}} \\

\bottomrule
\end{tabular}
}
\end{table}

The results can be seen in Table~\ref{tab:PSAVsSA}, where comparisons are made for the uncertainty levels 0.6 and 0.9. We see that there is a clear advantage of the population, and the advantage increases with the higher uncertainty levels. 

\subsection{Comparisons on the Deterministic Problem} \label{sec:single_comp}

APSA's advantages obtained from the population, a quicker cooling schedule and adaptive probabilistic selection for the perturbation of solutions, make for an effective algorithm in the uncertain setting. Its performance on the deterministic problem or RCJS is worth investigating, to see if any of the novelties proposed in the algorithm can be beneficial for the deterministic problem. The results are shown in table~\ref{tab:deterministic}, where APSA to ACS, SACS, and a column generation and ACO hybrid (CGACO) are compared. 

\begin{table*}
\centering
\caption{A comparison of CGACO \cite{Thiruvady2014CG}, ACS, SACS and PSA on RCJS. Statistical significance was obtained using a pair-wise Wilcoxon ranked-sum test, with significant results at a  95$\%$ confidence interval highlighted in bold.}
\label{tab:deterministic}
 \scalebox{0.75}{
\begin{tabular}{llrrrrrrrrrrrrrrr}
\toprule			
	&&		\multicolumn{3}{c}{CGACO}		&&	\multicolumn{3}{c}{ACS}	&&	\multicolumn{3}{c}{SACS} &&	\multicolumn{3}{c}{PSA}	\\	
	\cmidrule{3-5}\cmidrule{7-9}\cmidrule{11-13}\cmidrule{15-17}
		Instance	&&	Best	&	Mean	&	SD	&&	Best	&	Mean	&	SD	&&	Best	&		Mean		&	SD	&&	Best	&		Mean		&	SD	\\
  \midrule
Instance	&&	Best	&	Mean	&	SD	&&	Best	&	Mean	&	SD	&&	Best	&		Mean		&	SD	&&	Best	&		Mean		&	SD	\\
3	-	5	&&	515.68	&	520.30	&	6.16	&&	505.00	&	505.00	&	0.00	&&	505.00	&		505.00		&	0.00	&&	505.00	&		505.00		&	0.00	\\
3	-	23	&&	149.07	&	150.04	&	1.09	&&	149.07	&	149.09	&	0.08	&&	149.07	&		149.11		&	0.12	&&	149.07	&	{\bf	149.07}	&	0.00	\\
3	-	53	&&	69.36	&	69.36	&	0.00	&&	69.36	&	69.36	&	0.00	&&	69.36	&		69.36		&	0.00	&&	69.36	&		69.36		&	0.00	\\
4	-	28	&&	23.90	&	23.90	&	0.00	&&	23.94	&	24.04	&	0.09	&&	23.94	&		24.00		&	0.06	&&	23.81	&	{\bf	23.81}	&	0.00	\\
4	-	42	&&	68.59	&	71.26	&	3.01	&&	66.73	&	66.73	&	0.00	&&	66.73	&		66.73		&	0.00	&&	66.07	&	{\bf	66.36}	&	0.33	\\
4	-	61	&&	47.69	&	47.98	&	0.23	&&	45.96	&	45.96	&	0.00	&&	45.96	&		45.96		&	0.00	&&	45.96	&		45.96		&	0.00	\\
5	-	7	&&	260.15	&	275.00	&	17.65	&&	256.76	&	267.40	&	5.43	&&	256.69	&		262.81		&	4.29	&&	252.90	&	{\bf	253.37}	&	0.18	\\
5	-	21	&&	177.27	&	177.27	&	0.00	&&	168.63	&	169.21	&	0.92	&&	168.63	&		168.79		&	0.39	&&	168.63	&	{\bf	168.63}	&	0.00	\\
5	-	62	&&	264.99	&	276.22	&	13.75	&&	259.15	&	266.35	&	5.17	&&	259.12	&		266.37		&	5.43	&&	250.51	&	{\bf	252.63}	&	2.62	\\
\midrule																															
6	-	10	&&	1031.40	&	1031.40	&	0.00	&&	874.90	&	891.58	&	9.87	&&	870.83	&		887.66		&	8.29	&&	831.23	&	{\bf	843.86}	&	6.11	\\
6	-	28	&&	241.28	&	241.28	&	0.00	&&	230.04	&	239.48	&	5.35	&&	225.53	&		236.04		&	5.87	&&	218.37	&	{\bf	222.03}	&	1.91	\\
6	-	58	&&	264.33	&	268.87	&	4.72	&&	248.01	&	255.63	&	3.34	&&	241.04	&		251.53		&	4.52	&&	238.84	&	{\bf	240.65}	&	0.60	\\
7	-	5	&&	488.17	&	502.07	&	20.37	&&	445.18	&	460.07	&	6.51	&&	445.82	&		459.88		&	5.91	&&	418.06	&	{\bf	429.25}	&	7.87	\\
7	-	23	&&	655.52	&	658.40	&	2.55	&&	587.96	&	604.99	&	10.14	&&	577.88	&		595.26		&	8.20	&&	553.91	&	{\bf	570.05}	&	8.24	\\
7	-	47	&&	543.82	&	547.21	&	11.49	&&	472.24	&	502.76	&	12.45	&&	470.42	&		497.73		&	11.80	&&	428.51	&	{\bf	453.82}	&	12.17	\\
8	-	3	&&	911.22	&	1010.67	&	111.72	&&	725.49	&	768.23	&	15.77	&&	721.80	&		746.00		&	14.32	&&	622.92	&	{\bf	669.74}	&	17.23	\\
8	-	53	&&	522.59	&	522.58	&	0.00	&&	494.32	&	507.07	&	6.56	&&	487.27	&		501.04		&	6.56	&&	455.45	&	{\bf	470.12}	&	7.38	\\
8	-	77	&&	1610.28	&	1648.35	&	67.75	&&	1281.98	&	1311.84	&	14.88	&&	1265.41	&		1298.44		&	14.54	&&	1208.54	&	{\bf	1235.36}	&	14.73	\\
\midrule																															
9	-	20	&&	1017.43	&	1023.07	&	5.42	&&	999.95	&	1026.94	&	13.00	&&	991.70	&		1014.82		&	12.15	&&	932.37	&	{\bf	953.10}	&	13.71	\\
9	-	47	&&	1564.69	&	1715.26	&	178.01	&&	1290.84	&	1382.97	&	31.23	&&	1325.49	&		1362.21		&	14.96	&&	1226.35	&	{\bf	1284.73}	&	22.42	\\
9	-	62	&&	1730.74	&	1804.20	&	95.07	&&	1549.98	&	1578.38	&	15.15	&&	1546.42	&		1572.67		&	11.82	&&	1449.69	&	{\bf	1490.92}	&	20.10	\\
10	-	7	&&	2979.79	&	3108.48	&	125.12	&&	2731.54	&	2806.54	&	32.56	&&	2709.76	&		2782.76		&	27.75	&&	2567.60	&	{\bf	2657.63}	&	37.67	\\
10	-	13	&&	2573.10	&	2649.91	&	94.07	&&	2337.86	&	2389.39	&	26.48	&&	2304.22	&		2353.62		&	23.97	&&	2210.89	&	{\bf	2264.51}	&	29.62	\\
10	-	31	&&	674.48	&	674.48	&	0.00	&&	665.47	&	700.52	&	13.95	&&	663.87	&		702.88		&	15.53	&&	620.08	&	{\bf	639.33}	&	8.39	\\
11	-	21	&&	1100.04	&	1130.03	&	24.62	&&	1075.88	&	1132.02	&	22.93	&&	1075.24	&		1119.88		&	18.19	&&	1032.07	&	{\bf	1057.54}	&	15.66	\\
11	-	56	&&	2072.37	&	2216.39	&	192.02	&&	2009.59	&	2077.04	&	26.14	&&	1979.92	&		2044.09		&	29.91	&&	1872.61	&	{\bf	1946.54}	&	35.84	\\
11	-	63	&&	2360.96	&	2362.08	&	0.70	&&	2135.00	&	2193.68	&	28.26	&&	2132.93	&		2182.86		&	23.47	&&	2029.50	&	{\bf	2077.91}	&	23.15	\\
\midrule																															
12	-	14	&&	2257.65	&	2345.85	&	55.65	&&	1922.26	&	1983.88	&	35.70	&&	1851.20	&		1944.55		&	29.38	&&	1771.49	&	{\bf	1879.07}	&	31.24	\\
12	-	36	&&	3781.18	&	3966.80	&	209.41	&&	3192.39	&	3293.29	&	41.24	&&	3179.14	&		3250.00		&	36.68	&&	2992.22	&	{\bf	3110.58}	&	44.50	\\
12	-	80	&&	3116.71	&	3286.07	&	138.28	&&	2625.59	&	2697.86	&	35.31	&&	2557.22	&		2633.16		&	40.96	&&	2408.84	&	{\bf	2520.89}	&	52.05	\\
15	-	2	&&	5207.54	&	5353.60	&	266.67	&&	4220.06	&	4369.19	&	68.77	&&	4166.11	&		4276.08		&	58.83	&&	4041.59	&	{\bf	4209.96}	&	86.57	\\
15	-	3	&&	5538.80	&	6477.89	&	800.76	&&	4834.72	&	4967.70	&	65.56	&&	4694.78	&		4827.35		&	50.00	&&	4481.04	&	{\bf	4629.03}	&	89.13	\\
15	-	5	&&	4560.35	&	4942.42	&	649.77	&&	3794.91	&	3901.78	&	55.80	&&	3729.25	&		3825.77		&	56.47	&&	3496.33	&	{\bf	3612.27}	&	61.85	\\
20	-	2	&&	10959.00	&	10959.00	&	0.00	&&	9301.75	&	9536.87	&	104.95	&&	9076.52	&		9317.65		&	108.35	&&	8832.93	&	{\bf	9214.97}	&	171.43	\\
20	-	5	&&	20813.00	&	20813.00	&	0.00	&&	15595.70	&	16020.42	&	296.60	&&	15226.80	&		15664.50		&	244.66	&&	14590.60	&	{\bf	15304.66}	&	281.96	\\
20	-	6	&&	8789.68	&	9658.98	&	542.10	&&	8176.13	&	8371.80	&	100.00	&&	8057.20	&		8245.80		&	92.95	&&	7905.06	&	{\bf	8212.50}	&	159.31	\\

\bottomrule
\end{tabular}
}
\end{table*}

CGACO is also run 25 times on each problem instance. A single sample was used for ACS, SACS and APSA, leading to the original resource limit (effectively RCJS). We see in Table~\ref{tab:deterministic} that beyond small problem instances ($>$ 4 machines), APSA has a clear advantage. It generally performs best, but without the same large gains seen on RCJSU. 

\section{Conclusion}\label{sec:conclusion}
This study investigates an optimisation problem that originates in the mining industry, specifically transporting ore from mines to ports. Past studies have considered the problem as deterministic, however, in real settings, the resources needed to transport the ore are subject to uncertainty. The problem is also called resource constrained job scheduling with uncertainty, where for this study, we consider multiple uncertain scenarios. To tackle this problem, we propose an adaptive population-based simulated approach, which has significant advantages over previously proposed ACO based approaches. We find that using a population, increasing the rate of the cooling schedule within the Metropolis-Hastings algorithm and adaptively selecting intensification and diversification via the neighbourhood moves, proves advantageous. These modifications, compared to traditional simulated annealing, lead to finding robust solutions, and hence, best results on a benchmark dataset for RCJSU.

Despite the success of APSA on RCJSU, investigating how the approach performs on other optimisation problems with uncertainty will be of interest in future work. In particular, it is expected an adaptation of APSA is likely to be  effective for problems that consist of highly constrained disjoint feasible regions, such as RCJSU.  

The MIP model in Section~\ref{sec:problem} is clearly too complex and large for existing solvers, but decomposition approaches such as column generation or Benders' decomposition may prove effective. The advantage of such methods is that they can provide guarantees of solution quality (via lower bounds to RCJSU), which can help to measure the quality of the solutions found by APSA.

\bibliographystyle{IEEEtran}
\bibliography{References}

\end{document}